\documentclass[english,pdftex]{article}
\usepackage{helvet}
\usepackage[T1]{fontenc}
\usepackage[latin9]{inputenc}
\usepackage[pdftex]{geometry}
\usepackage{units}
\usepackage{amsmath}
\usepackage{amsthm}
\usepackage{amssymb}
\usepackage[pdftex]{graphicx}

\makeatletter
\theoremstyle{remark}
\newtheorem{rem}{\protect\remarkname}
\theoremstyle{definition}
\newtheorem{defn}{\protect\definitionname}
\theoremstyle{plain}
\newtheorem{prop}{\protect\propositionname}
\theoremstyle{plain}
\newtheorem{thm}{\protect\theoremname}
\theoremstyle{remark}
\newtheorem{claim}{\protect\claimname}
\theoremstyle{plain}
\newtheorem{cor}{\protect\corollaryname}
\theoremstyle{definition}
\newtheorem{example}{\protect\examplename}

\usepackage{iclr2023_conference,times}

\usepackage{amsmath,amsfonts,bm}

\def\eqref#1{equation~\ref{#1}}

\def\1{\bm{1}}

\DeclareMathAlphabet{\mathsfit}{\encodingdefault}{\sfdefault}{m}{sl}
\SetMathAlphabet{\mathsfit}{bold}{\encodingdefault}{\sfdefault}{bx}{n}

\usepackage[hyphens]{url}
\usepackage{hyperref}

\title{Implicit Bias of Large Depth Networks:\\ a Notion of Rank for Nonlinear Functions}

\author{Arthur Jacot \\
Courant Institute of Mathematical Sciences\\
New York University\\
New York, NY 10012, USA \\
\texttt{arthur.jacot@nyu.edu}
}

\iclrfinalcopy

\@ifundefined{showcaptionsetup}{}{%
 \PassOptionsToPackage{caption=false}{subfig}}
\usepackage{subfig}
\makeatother

\usepackage{babel}
\providecommand{\definitionname}{Definition}
\providecommand{\propositionname}{Proposition}
\providecommand{\remarkname}{Remark}
\providecommand{\theoremname}{Theorem}
\providecommand{\claimname}{Claim}
\providecommand{\corollaryname}{Corollary}
\providecommand{\examplename}{Example}

\begin{document}
\maketitle
\begin{abstract} We show that the representation cost of fully connected neural networks with homogeneous nonlinearities - which describes the implicit bias in function space of networks with $L_2$-regularization or with losses such as the cross-entropy - converges as the depth of the network goes to infinity to a notion of rank over nonlinear functions. We then inquire under which conditions the global minima of the loss recover the `true' rank of the data: we show that for too large depths the global minimum will be approximately rank 1 (underestimating the rank); we then argue that there is a range of depths which grows with the number of datapoints where the true rank is recovered. Finally, we discuss the effect of the rank of a classifier on the topology of the resulting class boundaries and show that autoencoders with optimal nonlinear rank are naturally denoising.
\end{abstract}
\renewcommand\cite{\citep}

\section{Introduction}

There has been a lot of recent interest in the so-called implicit
bias of DNNs, which describes what functions are favored by a network
when fitting the training data. Different network architectures (choice
of nonlinearity, depth, width of the network, and more) and training
procedures (initialization, optimization algorithm, loss) can lead
to widely different biases.

In contrast to the so-called kernel regime where the implicit bias
is described by the Neural Tangent Kernel \cite{jacot2018neural},
there are several active regimes (also called rich or feature-learning
regimes), whose implicit bias often feature a form sparsity that is
absent from the kernel regime. Such active regimes have been observed
for example in DNNs with small initialization \cite{Chizat2018,Rotskoff2018,li2020towards,jacot-2021-DLN-Saddle},
with $L_{2}$-regularization \cite{Savarese_2019_repres_bounded_norm_shallow_ReLU_net_1D,Ongie_2020_repres_bounded_norm_shallow_ReLU_net,jacot_2022_L2_reformulation}
or when trained on exponentially decaying losses \cite{gunasekar_2018_implicit_bias,gunasekar_2018_implicit_bias2,soudry2018implicit,du_2018_implicit_bias,Ji_2018_directional,chizat_2020_implicit_bias,Ji_2020_directional2}.
In the latter two cases, the implicit bias is described by the representation
cost:
\[
R(f)=\min_{\mathbf{W}:f_{\mathbf{W}}=f}\left\Vert \mathbf{W}\right\Vert ^{2}
\]
where $f$ is a function that can be represented by the network and
the minimization is over all parameters $\mathbf{W}$ that result
in a network function $f_{\mathbf{W}}$ equal to $f$, the parameters
$\mathbf{W}$ form a vector and $\left\Vert \mathbf{W}\right\Vert $
is the $L_{2}$-norm.

The representation cost can in some cases be explicitly computed for
linear networks. For diagonal linear networks, the representation
cost of a linear function $f(x)=w^{T}x$ equals the $L_{p}$ norm
$R(f)=L\left\Vert w\right\Vert _{p}^{p}$ of the vector $v$ for $p=\frac{2}{L}$
\cite{gunasekar_2018_implicit_bias,Moroshko_2020_implicit_bias_diag}
where $L$ is the depth of the network. For fully-connected linear
networks, the representation cost of a linear function $f(x)=Ax$
equals the $L_{p}$-Schatten norm (the $L_{p}$ norm of the singular
values) $R(f)=L\left\Vert A\right\Vert _{p}^{p}$ \cite{dai_2021_repres_cost_DLN}.

A common thread between these examples is a bias towards some notion
of sparsity: sparsity of the entries of the vector $w$ in diagonal
networks and sparsity of the singular values in fully connected networks.
Furthermore, this bias becomes stronger with depth and in the infinite
depth limit $L\to\infty$ the rescaled representation cost $R(f)/L$
converges to the $L_{0}$ norm $\left\Vert w\right\Vert _{0}$ (the
number of non-zero entries in $w$) in the first case and to the rank
$\mathrm{Rank}(A)$ in the second.

For shallow ($L=2$) nonlinear networks with a homogeneous activation,
the representation cost also takes the form of a $L_{1}$ norm \cite{bach2017_F1_norm,chizat_2020_implicit_bias,Ongie_2020_repres_bounded_norm_shallow_ReLU_net},
leading to sparsity in the effective number of neurons in the hidden
layer of the network.

However, the representation cost of deeper networks does not resemble
any typical norm ($L_{p}$ or not), though it still leads to some
form of sparsity \cite{jacot_2022_L2_reformulation}. Despite the
absence of explicit formula, we will show that the rescaled representation
cost $R(f)/L$ converges to some notion of rank in nonlinear networks
as $L\to\infty$, in analogy to infinite depth linear networks.

\subsection{Contributions}

We first introduce two notions of rank: the Jacobian rank $\mathrm{Rank}_{J}(f)=\max_{x}\mathrm{Rank}\left[Jf(x)\right]$
and the Bottleneck rank $\mathrm{Rank}_{BN}(f)$ which is the smallest
integer $k$ such that $f$ can be factorized $f=h\circ g$ with inner
dimension $k$. In general, $\mathrm{Rank}_{J}(f)\leq\mathrm{Rank}_{BN}(f)$,
but for functions of the form $f=\psi\circ A\circ\phi$ (for a linear
map $A$ and two bijections $\psi$ and $\phi$), we have $\mathrm{Rank}_{J}(f)=\mathrm{Rank}_{BN}(f)=\mathrm{Rank}A$.
These two notions of rank satisfy the properties (1) $\mathrm{Rank}f\in\mathbb{Z}$;
(2) $\mathrm{Rank}(f\circ g)\leq\min\{\mathrm{Rank}f,\mathrm{Rank}g\}$;
(3) $\mathrm{Rank}(f+g)\leq\mathrm{Rank}f+\mathrm{Rank}g$; (4) $\mathrm{Rank}(x\mapsto Ax+b)=\mathrm{Rank}A$.

We then show that in the infinite depth limit $L\to\infty$ the rescaled
representation cost of DNNs with a general homogeneous nonlinearity
is sandwiched between the Jacobian and Bottleneck ranks:
\[
\mathrm{Rank}_{J}\left(f\right)\leq\lim_{L\to\infty}\frac{R(f)}{L}\leq\mathrm{Rank}_{BN}\left(f\right).
\]
Furthermore $\lim_{L\to\infty}R(f)$ satisfies properties (2-4) above.
We also conjecture that the limiting representation cost equals its
upper bound $\mathrm{Rank}_{BN}(f)$.

We then study how this bias towards low-rank functions translates
to finite but large depths. We first show that for large depths the
rescaled norm of the parameters $\nicefrac{\left\Vert \hat{\mathbf{W}}\right\Vert ^{2}}{L}$
at any global minimum $\hat{\mathbf{W}}$ is upper bounded by $1+\nicefrac{C_{N}}{L}$
for a constant $C_{N}$ which depends on the training points. This
implies that the resulting function has approximately rank $1$ w.r.t.
the Jacobian and Bottleneck ranks. 

This is however problematic if we are trying to fit a `true function'
$f^{*}$ whose `true rank' $k=\mathrm{Rank}_{BN}f^{*}$ is larger
than 1. Thankfully we show that if $k>1$ the constant $C_{N}$ explodes
as $N\to\infty$, so that the above bound ($\nicefrac{\left\Vert \hat{\mathbf{W}}\right\Vert ^{2}}{L}\leq1+\nicefrac{C_{N}}{L}$)
is relevant only for very large depths when $N$ is large. We show
another upper bound $\nicefrac{\left\Vert \hat{\mathbf{W}}\right\Vert ^{2}}{L}\leq k+\nicefrac{C}{L}$
with a constant $C$ independent of $N$, suggesting the existence
of a range of intermediate depths where the network recovers the true
rank $k$.

Finally, we discuss how rank recovery affects the topology of decision
boundaries in classification and leads autoencoders to naturally be
denoising, which we confirm with numerical experiments.

\subsection{Related Works}

The implicit bias of deep homogeneous networks has, to our knowledge,
been much less studied than those of either linear networks or shallow
nonlinear ones. \cite{ongie2022_linear_layer_in_DNN} study deep networks
with only one nonlinear layer (all others being linear). Similarly
\cite{le_2022_IB_mix_linear_homogeneous} show a low-rank alignment
phenomenon in a network whose last layers are linear. 

Closer to our setup is the analysis of the representation cost of
deep homogeneous networks in \cite{jacot_2022_L2_reformulation},
which gives two reformulations for the optimization in the definition
of the representation cost, with some implications on the sparsity
of the representations, though the infinite depth limit is not studied.

A very similar analysis of the sparsity effect of large depth on the
global minima of $L_{2}$-regularized networks is given in \cite{timor_2022_implicit_large_depth},
however, they only show how the optimal weight matrices are almost
rank 1 (and only on average), while we show low-rank properties of
the learned function, as well as the existence of a layer with almost
rank 1 hidden representations.

\section{Preliminaries}

In this section, we define fully-connected DNNs and their representation
cost.

\subsection{Fully Connected DNNs}

In this paper, we study fully connected DNNs with $L+1$ layers numbered
from $0$ (input layer) to $L$ (output layer). Each layer $\ell\in\{0,\dots,L\}$
has $n_{\ell}$ neurons, with $n_{0}=d_{in}$ the input dimension
and $n_{L}=d_{out}$ the output dimension. The pre-activations $\tilde{\alpha}_{\ell}(x)\in\mathbb{R}^{n_{\ell}}$
and activations $\alpha_{\ell}(x)\in\mathbb{R}^{n_{\ell}}$ of the
layers of the network are defined inductively as
\begin{align*}
\alpha_{0}(x) & =x\\
\tilde{\alpha}_{\ell}(x) & =W_{\ell}\alpha_{\ell-1}(x)+b_{\ell}\\
\alpha_{\ell}(x) & =\sigma\left(\tilde{\alpha}_{\ell}(x)\right),
\end{align*}
for the $n_{\ell}\times n_{\ell-1}$ connection weight matrix $W_{\ell}$,
the $n_{\ell}$ bias vector $b_{\ell}$ and the nonlinearity $\sigma:\mathbb{R}\to\mathbb{R}$
applied entrywise to the vector $\tilde{\alpha}_{\ell}(x)$. The parameters
of the network are the collection of all connection weights matrices
and bias vectors $\mathbf{W}=\left(W_{1},b_{1},\dots,W_{L},b_{L}\right)$.

We call the network function $f_{\mathbf{W}}:\mathbb{R}^{d_{in}}\to\mathbb{R}^{d_{out}}$
the function that maps an input $x$ to the pre-activations of the
last layer $\tilde{\alpha}_{L}(x)$.

In this paper, we will focus on homogeneous nonlinearities $\sigma$,
i.e. such that $\sigma(\lambda x)=\lambda\sigma(x)$ for any $\lambda\geq0$
and $x\in\mathbb{R}$, such as the traditional ReLU $\sigma(x)=\max\{0,x\}$.
In our theoretical analysis we will assume that the nonlinearity is
of the form $\sigma_{a}(x)=\begin{cases}
x & \text{if \ensuremath{x\geq0}}\\
ax & \text{otherwise}
\end{cases}$ for some $\alpha\in(-1,1)$, since for a general homogeneous nonlinearity
$\sigma$ (which is not proportional to the identity function, the
constant zero function or the absolute function), there are scalars
$a\in(-1,1)$, $b\in\mathbb{R}$ and $c\in\{+1,-1\}$ such that $\sigma(x)=c\sigma_{a}(bx)$;
as a result, the global minima and representation cost are the same
up to scaling. 
\begin{rem}
By a simple generalization of the work of \cite{arora_2018_relu_piecewise_lin},
the set of functions that can be represented by networks (with any
finite widths and depth) with such nonlinearities is the set of piecewise
linear functions with a finite number of linear regions. In contrast,
the three types of homogeneous nonlinearities we rule out (the identity,
the constant, or the absolute value) lead to different sets of functions:
the linear functions, the constant functions, or the piecewise linear
functions $f$ such that $\lim_{t\to\infty}\left\Vert f(tx)-f(-tx)\right\Vert $
is finite for all directions $x\in\mathbb{R}^{d_{in}}$ (or possibly
a subset of this class of functions). While some of the results of
this paper could probably be generalized to the third case up to a
few details, we rule it out for the sake of simplicity.
\end{rem}
\begin{rem}
All of our results will be for sufficiently wide networks, i.e. for
all widths $\mathbf{n}$ such that $n_{\ell}\geq n_{\ell}^{*}$ for
some minimal widths $n_{\ell}^{*}$. Moreover these results are $O(0)$
in the width, in the sense that above the threshold $n_{\ell}^{*}$
the constants do not depend on the widths $n_{\ell}$. When there
are a finite number of datapoints $N$, it was shown by \cite{jacot_2022_L2_reformulation}
that a width of $N(N+1)$ is always sufficient, that is we can always
take $n_{\ell}^{*}=N(N+1)$ (though it is observed empirically that
a much smaller width can be sufficient in some cases). When we are
trying to fit a piecewise linear function over the whole input domain
$\Omega$, the width required depends on the number of linear regions
\cite{he_2018_relu_piecewise_lin}.
\end{rem}

\subsection{Representation Cost}

The representation cost $R(f;\Omega,\sigma,L)$ is the squared norm
of the optimal weights $\mathbf{W}$ which represents the function
$f_{|\Omega}$:
\[
R(f;\Omega,\sigma,L)=\min_{\mathbf{W}:f_{\mathbf{W}|\Omega}=f_{|\Omega}}\left\Vert \mathbf{W}\right\Vert ^{2}
\]
where the minimum is taken over all weights $\mathbf{W}$ of a depth
$L$ network (with some finite widths $\mathbf{n}$) such that $f_{\mathbf{W}}(x)=f(x)$
for all $x\in\Omega$. If no such weights exist, we define $R(f;\Omega,\sigma,L)=\infty$. 

The representation cost describes the natural bias on the represented
function $f_{\mathbf{W}}$ induced by adding $L_{2}$ regularization
on the weights $\mathbf{W}$:
\[
\min_{\mathbf{W}}C(f_{\mathbf{W}})+\lambda\left\Vert \mathbf{W}\right\Vert ^{2}=\min_{f}C(f)+\lambda R(f;\Omega,\sigma,L)
\]
for any cost $C$ (defined on functions $f:\Omega\mapsto\mathbb{R}^{d_{out}}$)
and where the minimization on the right is over all functions $f$
that can be represented by a depth $L$ network with nonlinearity
$\sigma$. Therefore, if we can give a simple description of the representation
cost of a function $f$, we can better understand what type of functions
$f$ are favored by a DNN with nonlinearity $\sigma$ and depth $L$.
\begin{rem}
Note that the representation cost does not only play a role in the
presence of $L_{2}$-regularization, it also describes the implicit
bias of networks trained on an exponentially decaying loss, such as
the cross-entropy loss, as described in \cite{soudry2018implicit,gunasekar_2018_implicit_bias,chizat_2020_implicit_bias}.
\end{rem}

\section{Infinitely Deep Networks}

In this section, we first give 4 properties that a notion of rank
on piecewise linear functions should satisfy and introduce two notions
of rank that satisfy these properties. We then show that the infinite-depth
limit $L\to\infty$ of the rescaled representation cost $R(f;\Omega,\sigma_{a},L)/L$
is sandwiched between the two notions of rank we introduced, and that
this limit satisfies 3 of the 4 properties we introduced.

\subsection{Rank of Piecewise Linear Functions}

There is no single natural definition of rank for nonlinear functions,
but we will provide two of them in this section and compare them.
We focus on notions of rank for piecewise linear functions with a
finite number of linear regions since these are the function that
can be represented by DNNs with homogeneous nonlinearities (this is
a Corollary of Theorem 2.1 from \cite{arora_2018_relu_piecewise_lin},
for more details, see Appendix E.1). We call such functions finite
piecewise linear functions (FPLF).

Let us first state a set of properties that any notion of rank on
FPLF should satisfy, inspired by properties of rank for linear functions:
\begin{enumerate}
\item The rank of a function is an integer $\mathrm{Rank}(f)\in\mathbb{N}$.
\item $\mathrm{Rank}(f\circ g)\leq\min\{\mathrm{Rank}f,\mathrm{Rank}g\}$.
\item $\mathrm{Rank}(f+g)\leq\mathrm{Rank}f+\mathrm{Rank}g$.
\item If $f$ is affine ($f(x)=Ax+b$) then $\mathrm{Rank}f=\mathrm{Rank}A$.
\end{enumerate}
Taking $g=id$ or $f=id$ in (2) implies $\mathrm{Rank}(f)\leq\min\{d_{in},d_{out}\}$.
Properties (2) and (4) also imply that for any bijection $\phi$ on
$\mathbb{R}^{d}$, $\mathrm{Rank}(\phi)=\mathrm{Rank}(\phi^{-1})=d$.

Note that these properties do not uniquely define a notion of rank.
Indeed we will now give two notions of rank which satisfy these properties
but do not always match. However any such notion of rank must agree
on a large family of functions: Property 2 implies that $\mathrm{Rank}$
is invariant under pre- and post-composition with bijections (see
Appendix A), which implies that the rank of functions of the form
$\psi\circ f\circ\phi$ for an affine function $f(x)=Ax+b$ and two
(piecewise linear) bijections $\psi$ and $\phi$ always equals $\mathrm{Rank}A$.

The first notion of rank we consider is based on the rank of the Jacobian
of the function:
\begin{defn}
The Jacobian rank of a FPLF $f$ is $\mathrm{Rank}_{J}(f;\Omega)=\max_{x}\mathrm{Rank}Jf(x)$,
taking the max over points where $x$ is differentiable.
\end{defn}
Note that since the jacobian is constant over the linear regions of
the FPLF $f$, we only need to take the maximum over every linear
region. As observed in \cite{feng_2022_Jacobian_rank_DNN}, the Jacobian
rank measures the intrinsic dimension of the output set $f(\Omega)$.

The second notion of rank is inspired by the fact that for linear
functions $f$, the rank of $f$ equals the minimal dimension $k$
such that $f$ can be written as the composition of two linear function
$f=g\circ h$ with inner dimension $k$. We define the bottleneck
rank as:
\begin{defn}
The bottleneck rank $\mathrm{Rank}_{BN}(f;\Omega)$ is the smallest
integer $k\in\mathbb{N}$ such that there is a factorization as the
composition of two FPLFs $f_{|\Omega}=\left(g\circ h\right)_{|\Omega}$
with inner dimension $k$.
\end{defn}
The following proposition relates these two notions of rank:
\begin{prop}
Both $\mathrm{Rank}_{J}$ and $\mathrm{Rank}_{BN}$ satisfy properties
$1-4$ above. Furthermore:
\begin{itemize}
\item For any FPLF and any set $\Omega$, $\mathrm{Rank}_{J}(f;\Omega)\leq\mathrm{Rank}_{BN}(f;\Omega).$
\item There exists a FPLF $f:\mathbb{R}^{2}\to\mathbb{R}^{2}$ and a domain
$\Omega$ such that $\mathrm{Rank}_{J}(f;\Omega)=1$ and $\mathrm{Rank}_{BN}(f;\Omega)=2$.
\end{itemize}
\end{prop}

\subsection{Infinite-depth representation cost}

In the infinite-depth limit, the (rescaled) representation cost of
DNNs $R_{\infty}(f;\Omega,\sigma_{a})=\lim_{L\to\infty}\frac{R(f;\Omega,\sigma_{a},L)}{L}$
converges to a value `sandwiched' between the above two notions of
rank:
\begin{thm}
\label{thm:infinite_depth_representation_cost_sandwich_bound}For
any bounded domain $\Omega$ and any FPLF $f$
\[
\mathrm{Rank}_{J}(f;\Omega)\leq R_{\infty}(f;\Omega,\sigma_{\alpha})\leq\mathrm{Rank}_{BN}(f;\Omega).
\]
Furthermore the limiting representation cost $R_{\infty}(f;\Omega,\sigma_{a})$
satisfies properties 2 to 4.
\end{thm}
\begin{proof}
The lower bound follows from taking $L\to\infty$ in Proposition \ref{prop:param_norm_bound_Schatten_Jacobian}
(see Section \ref{subsec:Approximate-Rank-1-regime}). The upper bound
is constructive: a function $f=h\circ g$ can be represented as a
network in three consecutive parts: a first part (of depth $L_{g}$)
representing $g$, a final part (of depth $L_{h}$) representing $h$,
and in the middle $L-L_{g}-L_{h}$ identity layers on a $k$-dimensional
space. The contribution to the norm of the parameters of the middle
part is $k(L-L_{g}-L_{h})$ and it dominates as $L\to\infty$, since
the contribution of the first and final parts are finite.
\end{proof}
Note that $R_{\infty}(f;\Omega,\sigma_{a})$ might satisfy property
$1$ as well, we were simply not able to prove it. Theorem \ref{thm:infinite_depth_representation_cost_sandwich_bound}
implies that for functions of the form $f=\psi\circ A\circ\phi$ for
bijections $\psi$ and $\phi$, $R_{\infty}(f;\Omega,\sigma_{a})=\mathrm{Rank}_{J}(f;\Omega)=\mathrm{Rank}_{BN}(f;\Omega)=\mathrm{Rank}A$.
\begin{rem}
Motivated by some aspects of the proofs and a general intuition (which
is described in Section \ref{subsec:Discussion}) we conjecture that
$R_{\infty}(f;\Omega,\sigma_{a})=\mathrm{Rank}_{BN}(f;\Omega)$. This
would imply that the limiting representation cost does not depend
on the choice of nonlinearity, as long as it is of the form $\sigma_{a}$
(which we already proved is the case for functions of the form $\psi\circ A\circ\phi$).
\end{rem}
This result suggests that large-depth neural networks are biased towards
function which have a low Jacobian rank and (if our above mentioned
conjecture is true) low Bottleneck rank, much like linear networks
are biased towards low-rank linear maps. It also suggests that the
rescaled norm of the parameters $\nicefrac{\left\Vert \mathbf{W}\right\Vert ^{2}}{L}$
is an approximate upper bound on the Jacobian rank (and if our conjecture
is true on the Bottleneck rank too) of the function $f_{\mathbf{W}}$.
In the next section, we partly formalize these ideas.

\section{Rank Recovery in Finite Depth Networks}

In this section, we study how the (approximate) rank of minimizer
functions $f_{\hat{\mathbf{W}}}$ (i.e. functions at a global minimum
$\hat{\mathbf{W}}$) for the MSE $\mathcal{L}_{\lambda}(\mathbf{W})=\frac{1}{N}\sum_{i=1}^{N}(f_{\mathbf{W}}(x_{i})-y_{i})^{2}+\frac{\lambda}{L}\left\Vert \mathbf{W}\right\Vert ^{2}$
with data sampled from a distribution with support $\Omega$ is affected
by the depth $L$. In particular, when the outputs are generated from
a true function $f^{*}$ (i.e. $y_{i}=f^{*}(x_{i})$) with $k=\mathrm{Rank}_{BN}(f^{*};\Omega)$,
we study in which condition the `true rank' $k$ is recovered.

\subsection{Approximate Rank 1 Regime \label{subsec:Approximate-Rank-1-regime}}

One can build a function with BN-rank 1 that fits any training data
(for example by first projecting the input to a line with no overlap
and then mapping the points from the line to the outputs with a piecewise
linear function). This implies the following bound:
\begin{prop}
\label{prop:almost_rank_1_very_deep}There is a constant $C_{N}$
(which depends on the training data only) such that for any large
enough $L$, at any global minimum $\hat{\mathbf{W}}$ of the loss
$\mathcal{L}_{\lambda}$ the represented function $f_{\hat{\mathbf{W}}}$
satisfies
\[
\frac{1}{L}R(f_{\hat{\mathbf{W}}};\sigma_{a},\Omega,L)\leq1+\frac{C_{N}}{L}.
\]
\end{prop}
\begin{proof}
We use the same construction as in the proof of Theorem \ref{thm:infinite_depth_representation_cost_sandwich_bound}
for any fitting rank $1$ function.
\end{proof}
This bound implies that the function $f_{\hat{\mathbf{W}}}$ represented
by the network at a global minimum is approximately rank $1$ both
w.r.t. to the Jacobian and Bottleneck ranks, showing the bias towards
low-rank functions even for finite (but possibly very large) depths.

\textbf{Jacobian Rank: }For any function $f$, the rescaled norm representation
cost $\frac{1}{L}R(f;\Omega,\sigma_{a},L)$ bounds the $L_{p}$-Schatten
norm of the Jacobian (with $p=\frac{2}{L}$) at any point:
\begin{prop}
\label{prop:param_norm_bound_Schatten_Jacobian}Let $f$ be a FPLF,
then at any differentiable point $x$, we have
\[
\left\Vert Jf(x)\right\Vert _{\nicefrac{2}{L}}^{\nicefrac{2}{L}}:=\sum_{k=1}^{\mathrm{Rank}Jf_{\mathbf{W}}(x)}s_{k}\left(Jf(x)\right)^{\frac{2}{L}}\leq\frac{1}{L}R(f;\Omega,\sigma_{a},L),
\]
where $s_{k}\left(Jf_{\mathbf{W}}(x)\right)$ is the $k$-th singular
value of the Jacobian $Jf_{\mathbf{W}}(x)$.
\end{prop}
Together with Proposition \ref{prop:almost_rank_1_very_deep}, this
implies that the second singular value of the Jacobian of any minimizer
function must be exponentially small $s_{2}\left(Jf_{\hat{\mathbf{W}}}(x)\right)\leq\left(\frac{1+\frac{C_{N}}{L}}{2}\right)^{\frac{L}{2}}$
in $L$.

\textbf{Bottleneck Rank: }We can further prove the existence of a
bottleneck in the network in any minimizer network, i.e. a layer $\ell$
whose hidden representation is approximately rank 1:
\begin{prop}
\label{prop:almost_BN_rank_1}For any global minimum $\hat{\mathbf{W}}$
of the $L_{2}$-regularized loss $\mathcal{L}_{\lambda}$ with $\lambda>0$
and any set of $\tilde{N}$ datapoints $\tilde{X}\in\mathbb{R}^{d_{in}\times\tilde{N}}$
(which do not have to be the training set $X$) with non-constant
outputs, there is a layer $\ell_{0}$ such that the first two singular
values $s_{1},s_{2}$ of the hidden representation $Z_{\ell_{0}}\in\mathbb{R}^{n_{\ell}\times N}$
(whose columns are the activations $\alpha_{\ell_{0}}(x_{i})$ for
all the inputs $x_{i}$ in $\tilde{X}$) satisfies $\frac{s_{2}}{s_{1}}=O(L^{-\frac{1}{4}})$.
\end{prop}
The fact that the global minima of the loss are approximately rank
1 not only in the Jacobian but also in the Bottleneck sense further
supports our conjecture that the limiting representation cost equals
the Bottleneck rank $R_{\infty}=\mathrm{Rank}_{BN}$. Furthermore,
it shows that the global minimum of the $L_{2}$-regularized is biased
towards low-rank functions for large depths, since it fits the data
with (approximately) the smallest possible rank.

\subsection{Rank Recovery for Intermediate Depths}

However, learning rank 1 functions is not always a good thing. Assume
that we are trying to fit a `true function' $f^{*}:\Omega\to\mathbb{R}^{d_{out}}$
with a certain rank $k=\mathrm{Rank}_{BN}\left(f^{*};\Omega\right)$.
If $k>1$ the global minima of a large depth network will end up underestimating
the true rank $k$.

In contrast, in the linear setting underestimating the true rank is
almost never a problem: for example in matrix completion one always
wants to find a minimal rank solution \cite{candes2009exact,arora_2019_matrix_factorization}.
The difference is due to the fact that rank $1$ nonlinear functions
can fit any finite training set, which is not the case in the linear
case. 

Thankfully, for large datasets it becomes more and more difficult
to underestimate the rank, since for large $N$ fitting the data with
a rank 1 function requires large derivatives, which in turn implies
a large parameter norm:
\begin{thm}
\label{prop:explosion_norm_rank_1_fit}Given a Jacobian-rank $k$
true function $f^{*}:\Omega\to\mathbb{R}^{d_{out}}$ on a bounded
domain $\Omega$, then for all $\epsilon$ there is a constant $c_{\epsilon}$
such that for any BN-rank 1 function $\hat{f}:\Omega\to\mathbb{R}^{d_{out}}$
that fits $\hat{f}(x_{i})=f^{*}(x_{i})$ a dataset $x_{1},\dots,x_{N}$
sampled i.i.d. from a distribution $p$ with support $\Omega$, we
have $\frac{1}{L}R(\hat{f};\Omega,\sigma_{a},L)>c_{\epsilon}N^{\frac{2}{L}\left(1-\frac{1}{k}\right)}$
with prob. at least $1-\epsilon$.
\end{thm}
\begin{proof}
We show that there is a point $x\in\Omega$ with large derivative
$\left\Vert Jf(x)\right\Vert _{op}\geq\frac{\mathrm{TSP}(y_{1},\dots,y_{N})}{\mathrm{diam}(x_{1},\dots,x_{N})}$
for the Traveling Salesman Problem $\mathrm{TSP}(y_{1},\dots,y_{N})$,
i.e. the length of the shortest path passing through every point $y_{1},\dots,y_{m}$,
and the diameter $\mathrm{diam}(x_{1},\dots,x_{N})$ of the points
$x_{1},\dots,x_{N}$. This follows from the fact that the image of
$\hat{f}$ is a line going through all $y_{i}$s, and if $i$ and
$j$ are the first and last points visited, the image of segment $[x_{i},x_{j}]$
is a line from $y_{i}$ to $y_{j}$ passing through all $y_{k}$s.
The diameter is bounded by $\mathrm{diam}\Omega$ while the TSP scales
as $N^{1-\frac{1}{k}}$ \cite{beardwood_1959_shortest_path_many_points}
since the $y_{i}$s are sampled from a $k$-dimensional distribution.
The bound on the parameter norm then follows from Proposition \ref{prop:param_norm_bound_Schatten_Jacobian}.
\end{proof}
This implies that the constant $C_{N}$ in Proposition \ref{prop:almost_rank_1_very_deep}
explodes as the number of datapoints $N$ increases, i.e. as $N$
increases, larger and larger depths are required for the bound in
Proposition \ref{prop:almost_rank_1_very_deep} to be meaningful.
In that case, a better upper bound on the norm of the parameters can
be obtained, which implies that the functions $f_{\hat{\mathbf{W}}}$
at global minima are approximately rank $k$ or less (at least in
the Jacobian sense, according to Proposition \ref{prop:param_norm_bound_Schatten_Jacobian}):
\begin{prop}
\label{prop:almost_rank_k_global_min}Let the `true function' $f^{*}:\Omega\to\mathbb{R}^{d_{out}}$
be piecewise linear with $\mathrm{Rank}_{BN}(f^{*})=k$, then there
is a constant $C$ which depends on $f^{*}$ only such that any minimizer
function $f_{\hat{\mathbf{W}}}$ satisfies

\[
\frac{1}{L}R(f_{\hat{\mathbf{W}}};\sigma_{a},\Omega,L)\leq\frac{1}{L}R(f^{*};\sigma_{a},\Omega,L)\leq k+\frac{C}{L}.
\]
\end{prop}
Theorem \ref{prop:explosion_norm_rank_1_fit} and Proposition \ref{prop:almost_rank_k_global_min}
imply that if the number of datapoints $N$ is sufficiently large
($N>\left(\frac{k+\frac{C}{L}}{c}\right)^{\frac{kL}{2k-2}}$), there
are parameters $\mathbf{W}^{*}$ that fit the true function $f^{*}$
with a smaller parameter norm than any choice of parameters $\mathbf{W}$
that fit the data with a rank 1 function. In that case, the global
minima will not be rank 1 and might instead recover the true rank
$k$.

Another interpretation is that since the constant $C$ does not depend
on the number of training points $N$ (in contrast to $C_{N}$), there
is a range of depths (which grows as $N\to\infty$) where the upper
bound of Proposition \ref{prop:almost_rank_k_global_min} is below
that of Proposition \ref{prop:almost_rank_1_very_deep}. We expect
rank recovery to happen roughly in this range of depths: too small
depths can lead to an overestimation of the rank\footnote{Note that traditional regression models, such as Kernel Ridge Regression
(KRR) typically overestimate the true rank, as described in Appendix
D.1.}, while too large depths can lead to an underestimation.
\begin{rem}
Note that in our experiments, we were not able to observe gradient
descent converging to a solution that underestimates the true rank,
even for very deep networks. This is probably due to gradient descent
converging to one of the many local minima in the loss surface of
very deep $L_{2}$-regularized DNNs. Some recent theoretical results
offer a possible explanation for why gradient descent naturally avoids
rank 1 solutions: the proof of Proposition \ref{prop:explosion_norm_rank_1_fit}
shows that rank 1 fitting functions have exploding gradient as $N\to\infty$,
and such high gradient functions are known (at the moment only for
shallow networks with 1D inputs) to correspond to narrow minima \cite{mulayoff_2021_minima_stability_bias_1D}. 

Some of our results can be applied to local minima $\hat{\mathbf{W}}$with
a small norm: Proposition \ref{prop:param_norm_bound_Schatten_Jacobian}
implies that the Jacobian rank of $f_{\hat{\mathbf{W}}}$ is approximately
bounded by $\nicefrac{\left\Vert \hat{\mathbf{W}}\right\Vert ^{2}}{L}$.
Proposition \ref{prop:almost_BN_rank_1} also applies to local minima,
but only if $\nicefrac{\left\Vert \hat{\mathbf{W}}\right\Vert ^{2}}{L}\leq1+\nicefrac{C}{L}$
for some constant $C$, though it could be generalized.
\end{rem}
\begin{figure}
\vspace{-30bp}

\begin{center}

\subfloat{\includegraphics[viewport=0bp 0bp 422bp 320bp,clip,scale=0.3]{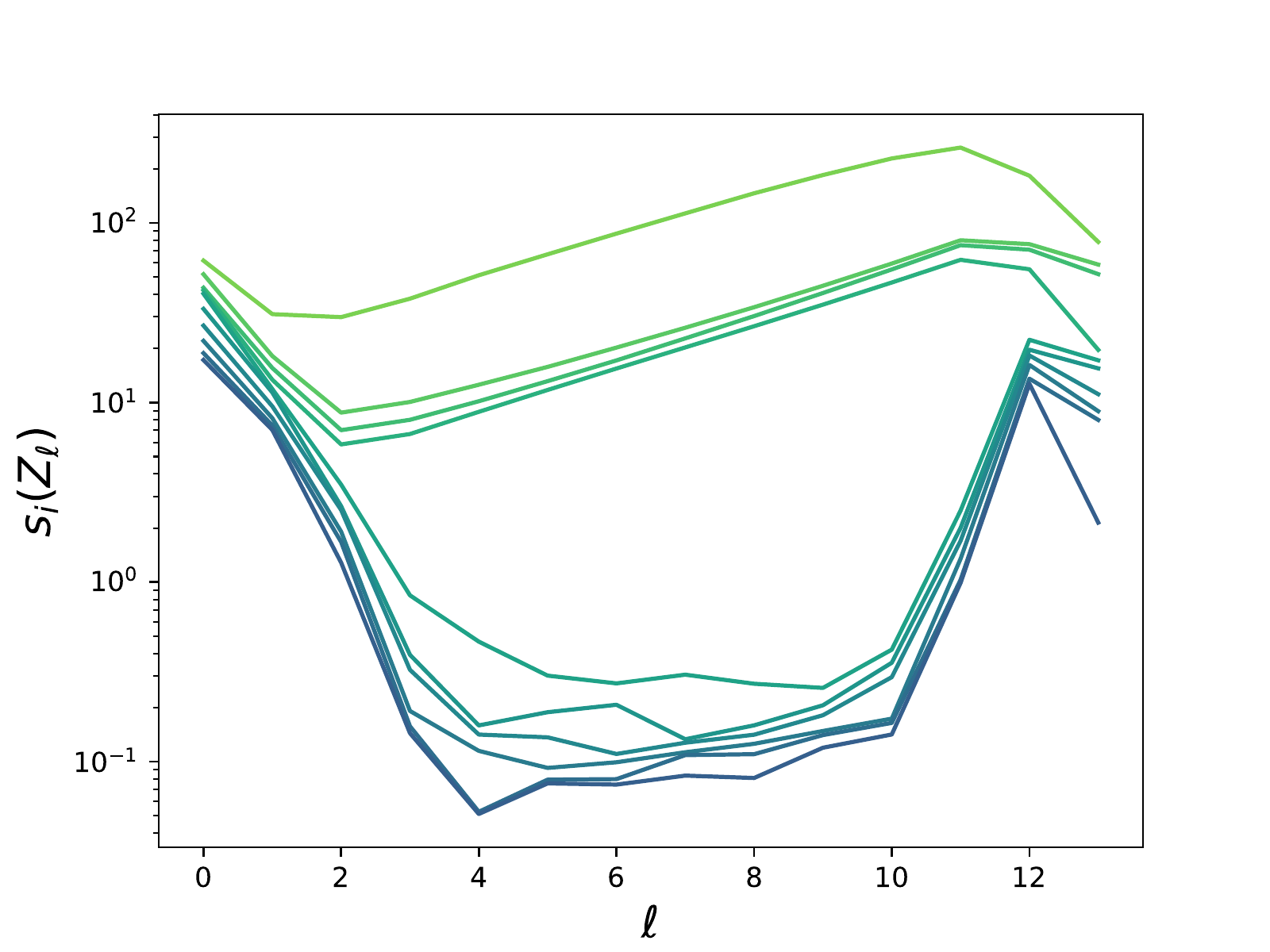}}\subfloat{\includegraphics[viewport=0bp 0bp 422bp 320bp,clip,scale=0.3]{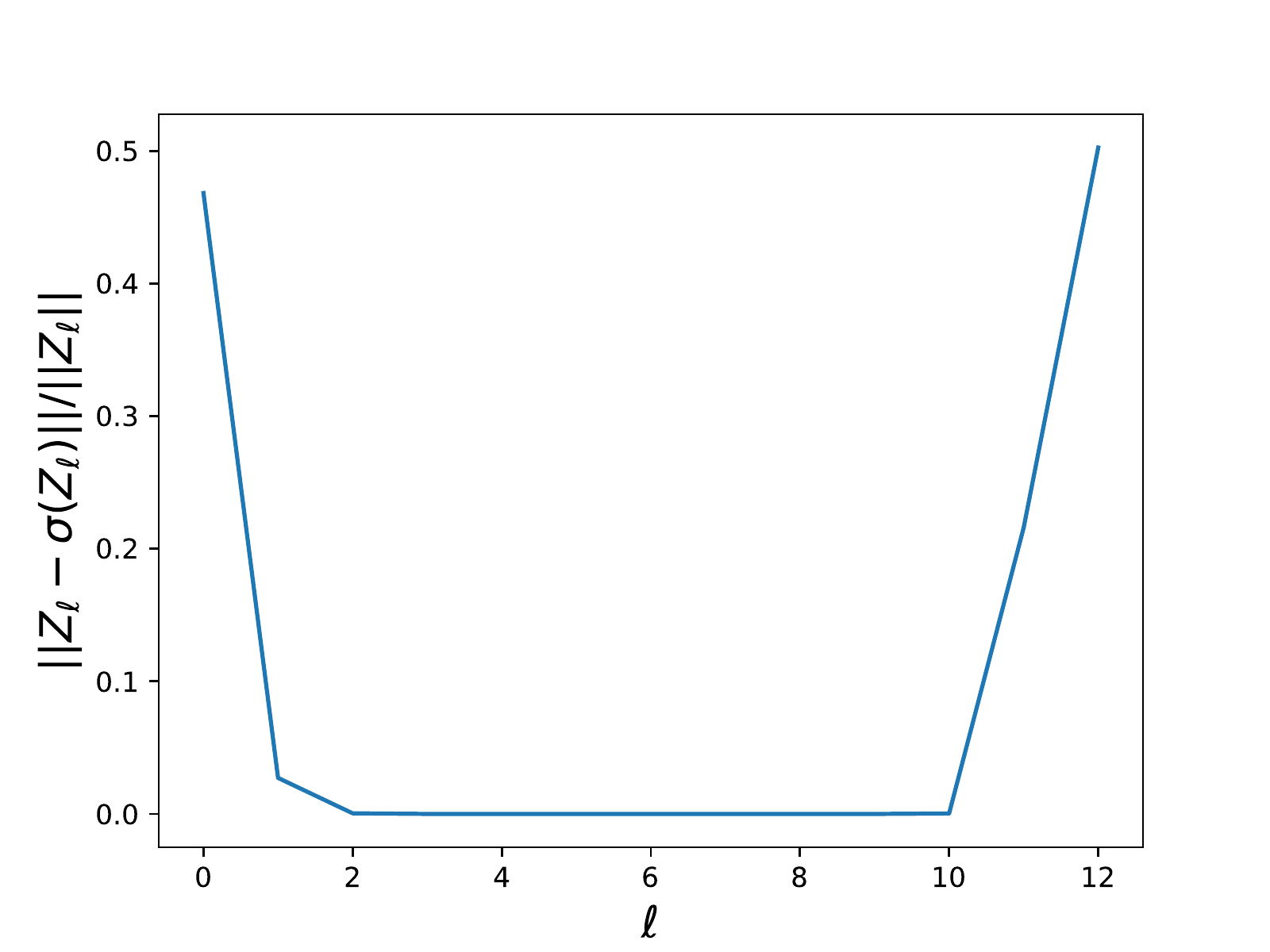}}\subfloat{\includegraphics[viewport=0bp 0bp 422bp 320bp,clip,scale=0.3]{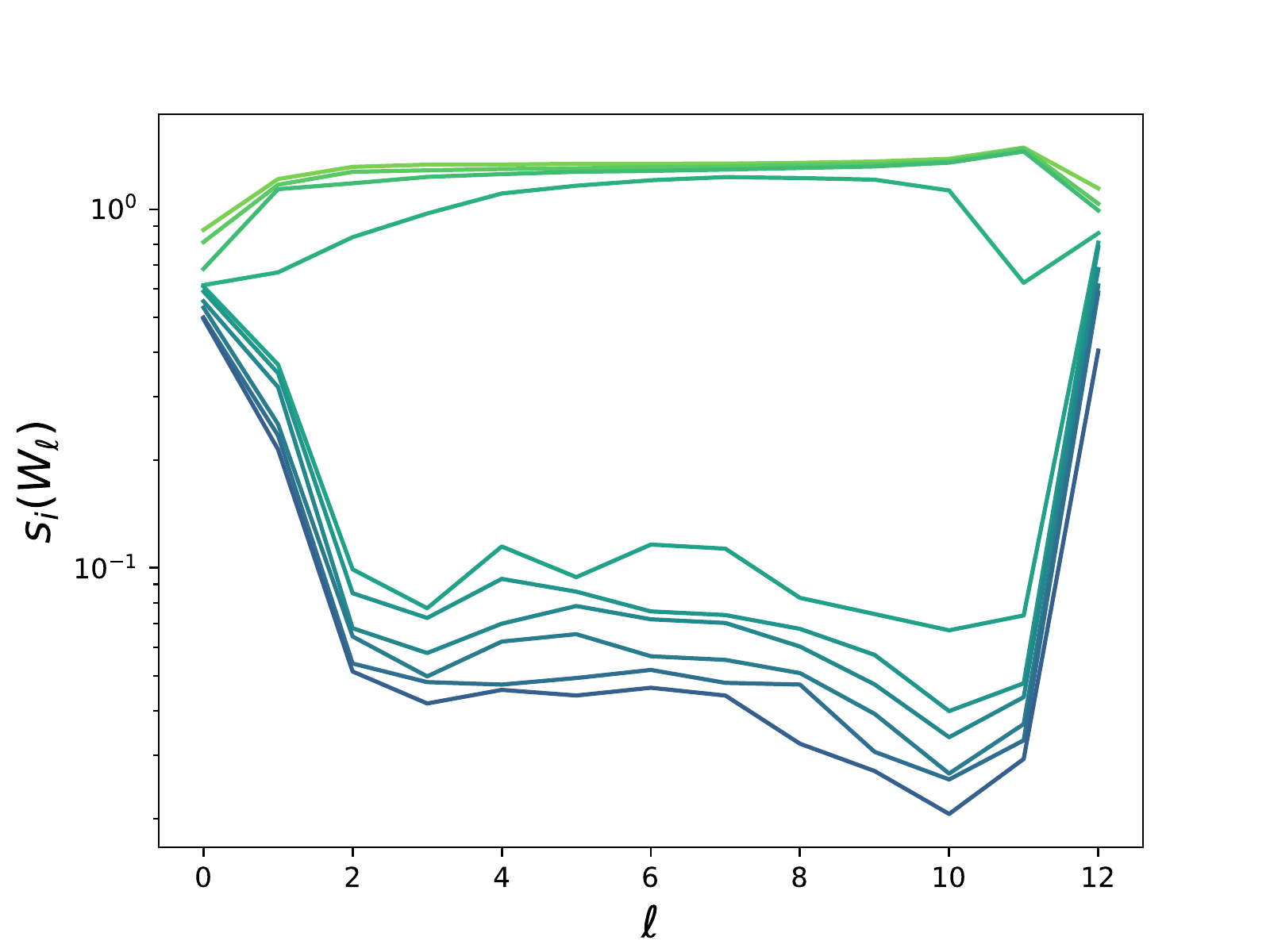}}

\end{center}

\caption{\label{fig:DNN_global_min_low_rank}DNN (depth $L=13$ and width $n_{\ell}=100$)
trained on a MSE task with rank $4$ true function $f^{*}:\mathbb{R}^{50}\to\mathbb{R}^{50}$,
with $N=500$ and $\lambda=0.05/L$. At the end of training, we obtain
$\nicefrac{\left\Vert W\right\Vert ^{2}}{L}\approx6$. \textbf{(left)}
First 10 singular values of the matrix of activations $Z_{\ell}$
for all $\ell$. The representations are appr. rank 4 in the middle
layers. \textbf{(middle)} The impact of the nonlinearity at each layer
$\ell$, measured by the ratio $\nicefrac{\left\Vert \tilde{Z}_{\ell}-Z_{\ell}\right\Vert _{F}}{\left\Vert \tilde{Z}_{\ell}\right\Vert _{F}}$
where $\tilde{Z}_{\ell}$ is the matrix of preactivations with entries
$\tilde{\alpha}_{k}^{(\ell)}(x_{i})$. This impact vanishes in the
middle layers, supporting our intuition that the middle layers represent
approximate identities. \textbf{(right)} First 10 singular values
of the weight matrices $W_\ell$ at every layer.}
\end{figure}

\subsection{Discussion\label{subsec:Discussion}}

We now propose a tentative explanation for the phenomenon observed
in this section. In contrast to the rest of the paper, this discussion
is informal.

Ideally, we want to learn functions $f$ which can be factorized as
a composition $h\circ g$ so that not only the inner dimension is
small but the two functions $g,h$ are not `too complex'. These two
objectives are often contradictory and one needs to find a trade-off
between the two. Instead of optimizing the bottleneck rank, one might
want to optimize with a regularization term of the form
\begin{equation}
\min_{f=h\circ g}k+\gamma\left(C(g)+C(h)\right),\label{eq:mixed_optimization}
\end{equation}
optimizing over all possible factorization $f=h\circ g$ of $f$ with
inner dimension $k$, where $C(g)$ and $C(h)$ are measures of the
complexity of $g$ and $h$ resp. The parameter $\gamma\geq0$ allows
us to tune the balance between the minimization of the inner dimension
and the complexity of $g$ and $h$, recovering the Bottleneck rank
when $\gamma=0$. For small $\gamma$ the minimizer is always rank
1 (since it is always possible to fit a finite dataset with a rank
$1$ function in the absence of restriction on the complexity on $g$
and $h$), but with the right choice of $\gamma$ one can recover
the true rank.

Some aspects of the proofs techniques we used in this paper suggest
that large-depth DNNs are optimizing such a cost (or an approximation
thereof). Consider a deep network that fits with minimal parameter
norm a function $f$; if we add more layers to the network it is natural
to assume that the new optimal representation of $f$ will be almost
the same as that of the shallower network with some added (approximate)
identity layers. The interesting question is where are those identity
layers added? The cost of adding an identity layer at a layer $\ell$
equals the dimension $d_{\ell}$ of the hidden representation of the
inputs at $\ell$. It is therefore optimal to add identity layers
where the hidden representations have minimal dimension.

This suggests that for large depths the optimal representation of
a function $f$ approximately takes the form of $L_{g}$ layers representing
$g$, then $L-L_{g}-L_{h}$ identity layers, and finally $L_{h}$
layers representing $h$, for some factorization $f=h\circ g$ with
inner dimension $k$. We observe in Figure \ref{fig:DNN_global_min_low_rank}
such a three-part representation structure in an MSE task with a low-rank
true function. The rescaled parameter norm would then take the form
\[
\frac{1}{L}\left\Vert \mathbf{W}\right\Vert ^{2}=\frac{L-L_{g}-L_{h}}{L}k+\frac{1}{L}\left(\left\Vert \mathbf{W}_{g}\right\Vert ^{2}+\left\Vert \mathbf{W}_{h}\right\Vert ^{2}\right),
\]
where $\mathbf{W}_{g}$ and $\mathbf{W}_{h}$ are the parameters of
the first and last part of the network. For large depths, we can make
the approximation $\frac{L-L_{g}-L_{h}}{L}\approx1$ to recover the
same structure as Equation \ref{eq:mixed_optimization}, with $\gamma=\nicefrac{1}{L}$,
$C(g)=\left\Vert \mathbf{W}\right\Vert _{g}^{2}$ and $C(h)=\left\Vert \mathbf{W}_{h}\right\Vert ^{2}$.
This intuition offers a possible explanation for rank recovery in
DNNs, though we are not yet able to prove it rigorously.

\section{Practical Implications}

In this section, we describe the impact of rank minimization on two
practical tasks: multiclass classification and autoencoders.

\subsection{Multiclass Classification}

Consider a function $f_{\mathbf{W}^{*}}:\mathbb{R}^{d_{in}}\to\mathbb{R}^{m}$
which solves a classification task with $m$ classes, i.e. for all
training points $x_{i}$ with class $y_{i}\in\{1,\dots,m\}$ the $y_{i}$-th
entry of the vector $f_{\mathbf{W}^{*}}$ is strictly larger than
all other entries. The Bottleneck rank $k=\mathrm{Rank}_{BN}(f_{\mathbf{W}^{*}})$
of $f_{\mathbf{W}^{*}}$ has an impact on the topology of the resulting
partition of the input space $\Omega$ into classes, leading to topological
properties typical of a partition on a $k$-dimensional space rather
than those of a partition on a $d_{in}$-dimensional space.

When $k=1$, the partition will be topologically equivalent to a classification
on a line, which implies the absence of tripoints, i.e. points at
the boundary of 3 (or more) classes. Indeed any boundary point $x\in\Omega$
will be mapped to a boundary point $z=g(x)$ by the first function
$g:\Omega\to\mathbb{R}$ in the factorization of $f_{\mathbf{W}^{*}}$;
since $z$ has at most two neighboring classes, then so does $x$.

This property is illustrated in Figure \ref{fig:Classification}:
for a classification task on four classes on the plane, we observe
that the partitions obtained by shallow networks ($L=2$) leads to
tripoints which are absent in deeper networks ($L=9$). Notice also
that the presence or absence of $L_{2}$-regularization has little
effect on the final shape, which is in line with the observation that
the cross-entropy loss leads to an implicit $L_{2}$-regularization
\cite{soudry2018implicit,gunasekar_2018_implicit_bias,chizat_2020_implicit_bias},
reducing the necessity of an explicit $L_{2}$-regularization.

\begin{figure}
\vspace{-45bp}

\begin{center}

\hspace{-60bp}\subfloat[$L=2,\lambda=0$]{\hspace{30bp}\includegraphics[viewport=65bp 10bp 335bp 310bp,clip,scale=0.35]{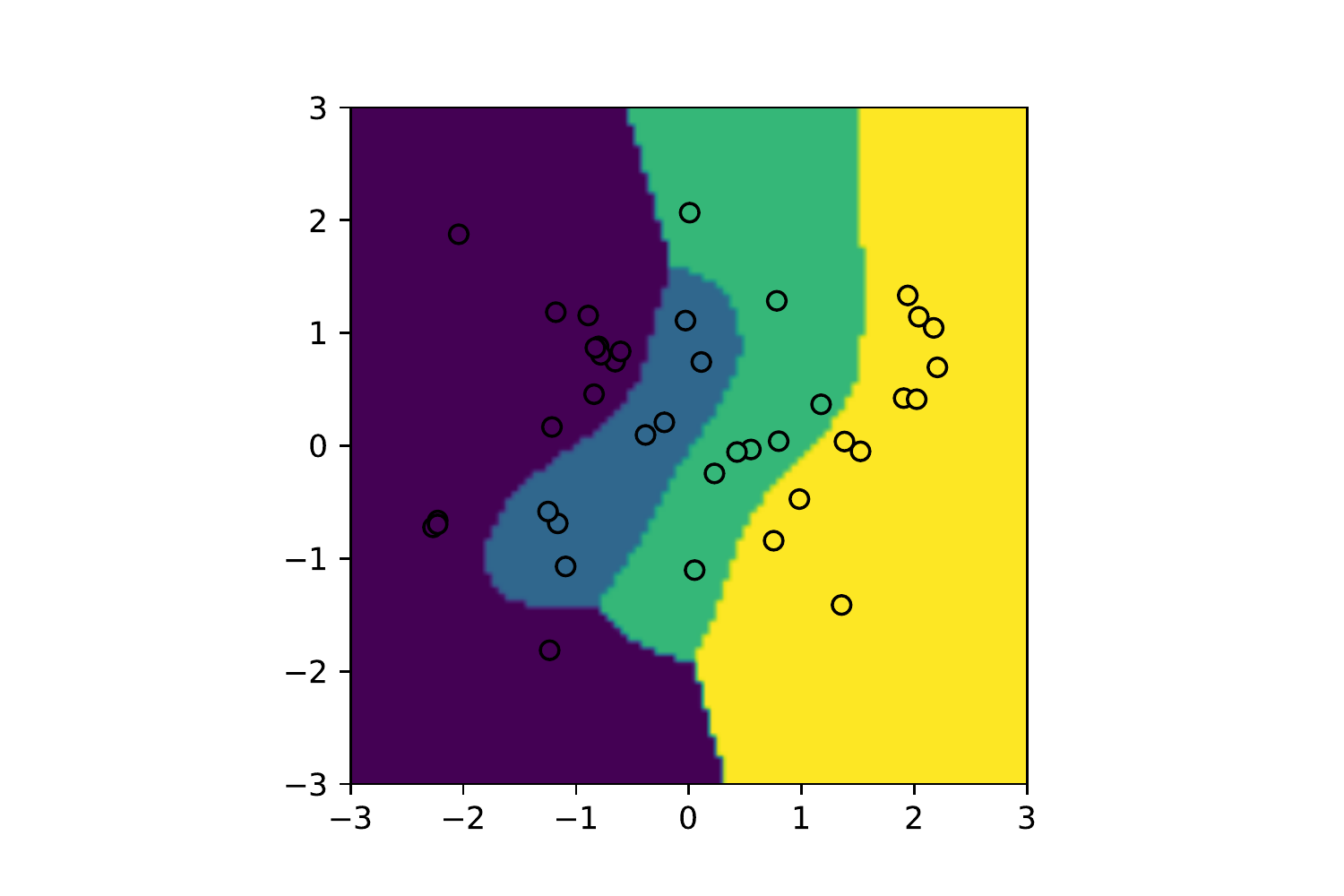}\hspace{30bp}}\hspace{-60bp}\subfloat[$L=2,\lambda=10^{-3}$]{\hspace{30bp}\includegraphics[viewport=65bp 10bp 335bp 310bp,clip,scale=0.35]{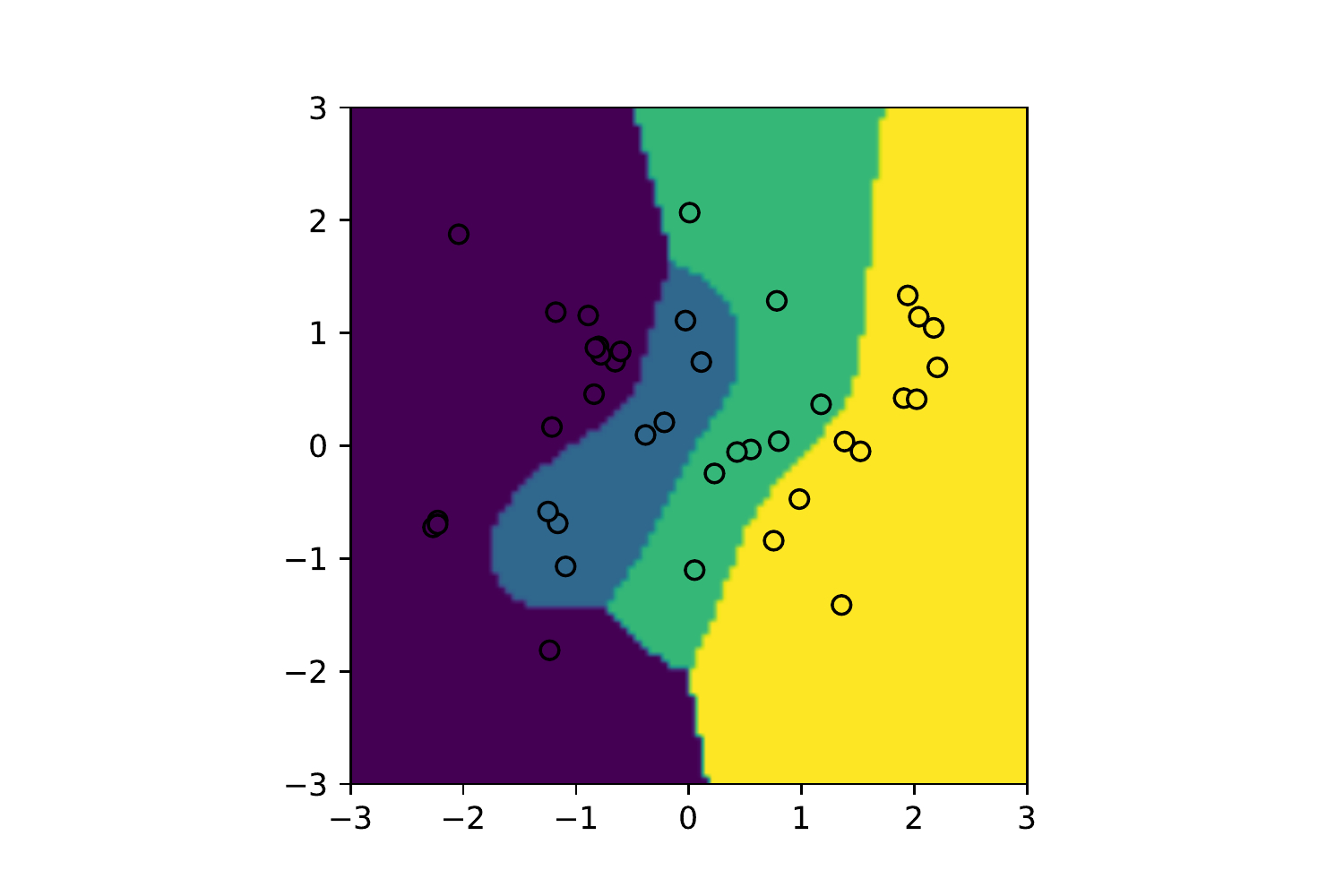}\hspace{30bp}

}\hspace{-60bp}\subfloat[$L=9,\lambda=0$]{\hspace{30bp}\includegraphics[viewport=65bp 10bp 335bp 310bp,clip,scale=0.35]{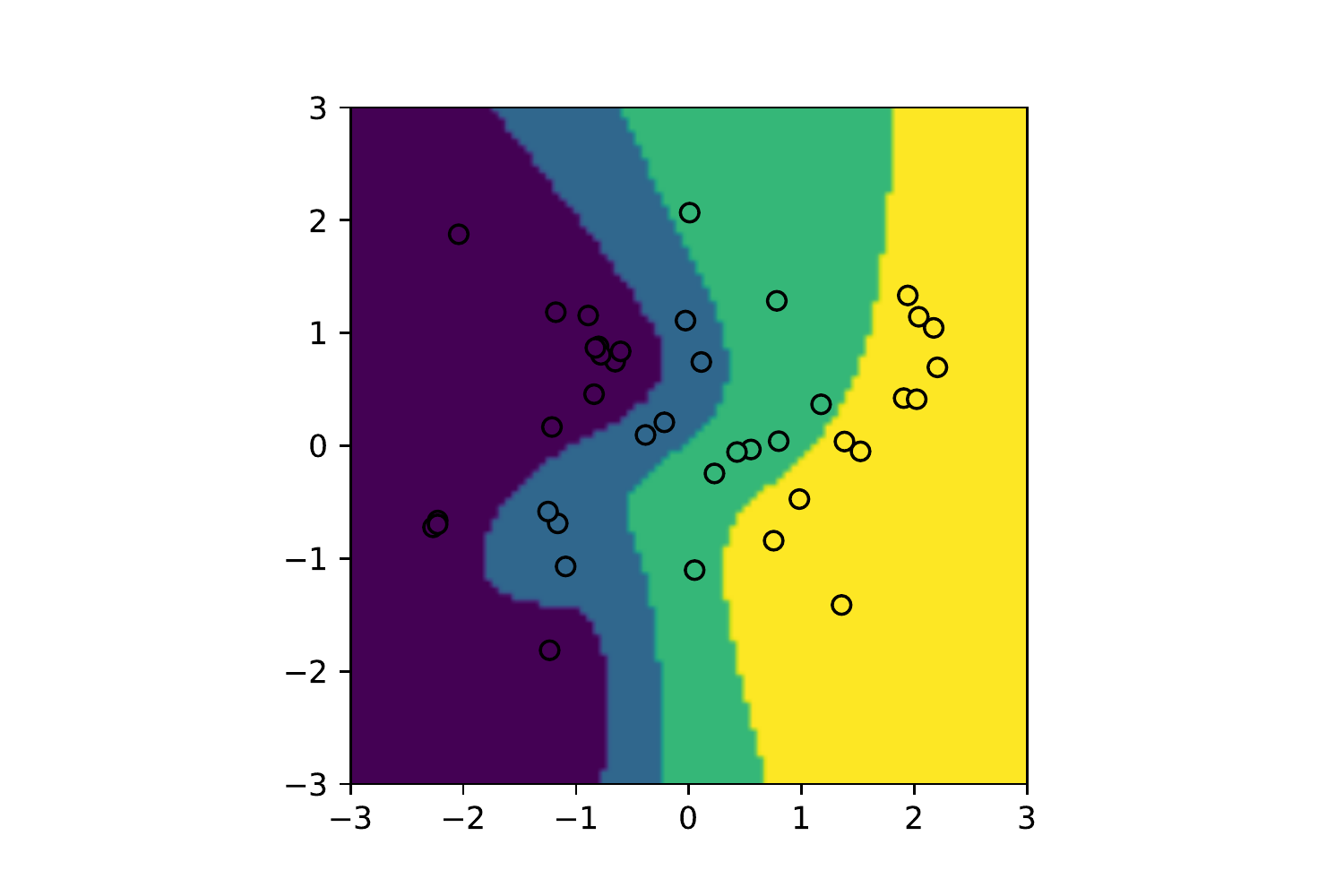}\hspace{30bp}

}\hspace{-60bp}\subfloat[$L=9,\lambda=10^{-3}$]{\hspace{30bp}\includegraphics[viewport=65bp 10bp 335bp 310bp,clip,scale=0.35]{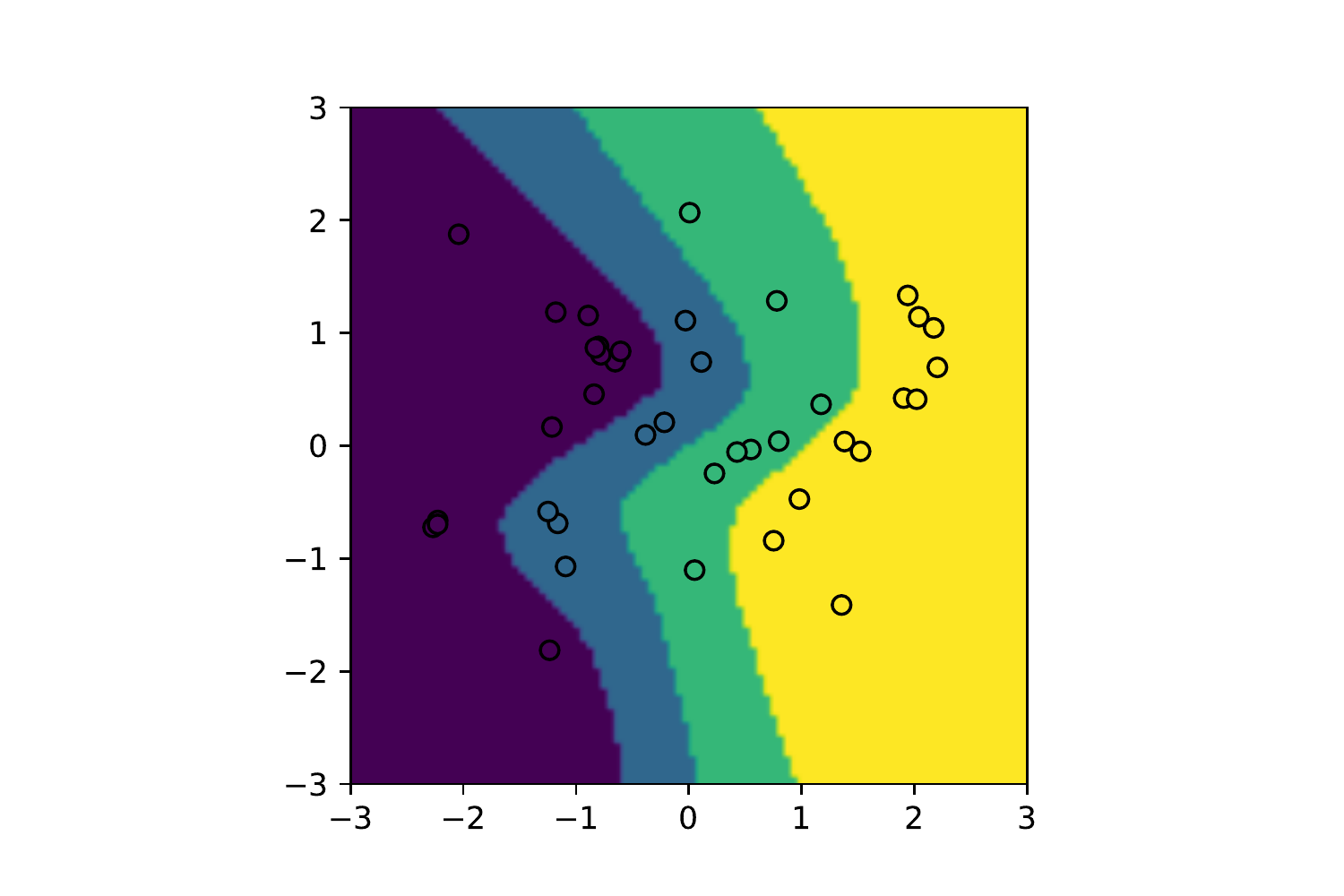}

\hspace{30bp}

}\hspace{-60bp}

\end{center}\caption{\label{fig:Classification}Classification on 4 classes (whose sampling
distribution are 4 identical inverted `S' shapes translated along
the $x$-axis) for two depths and with or without $L_{2}$-regularization.
The class boundaries in shallow networks \textbf{(A,B)} feature tripoints,
which are not observed in deeper networks \textbf{(C,D)}.}
\end{figure}

\begin{figure}
\vspace{-30bp}

\begin{center}

\subfloat[$L=5$]{\includegraphics[viewport=10bp 20bp 422bp 245bp,clip,scale=0.4]{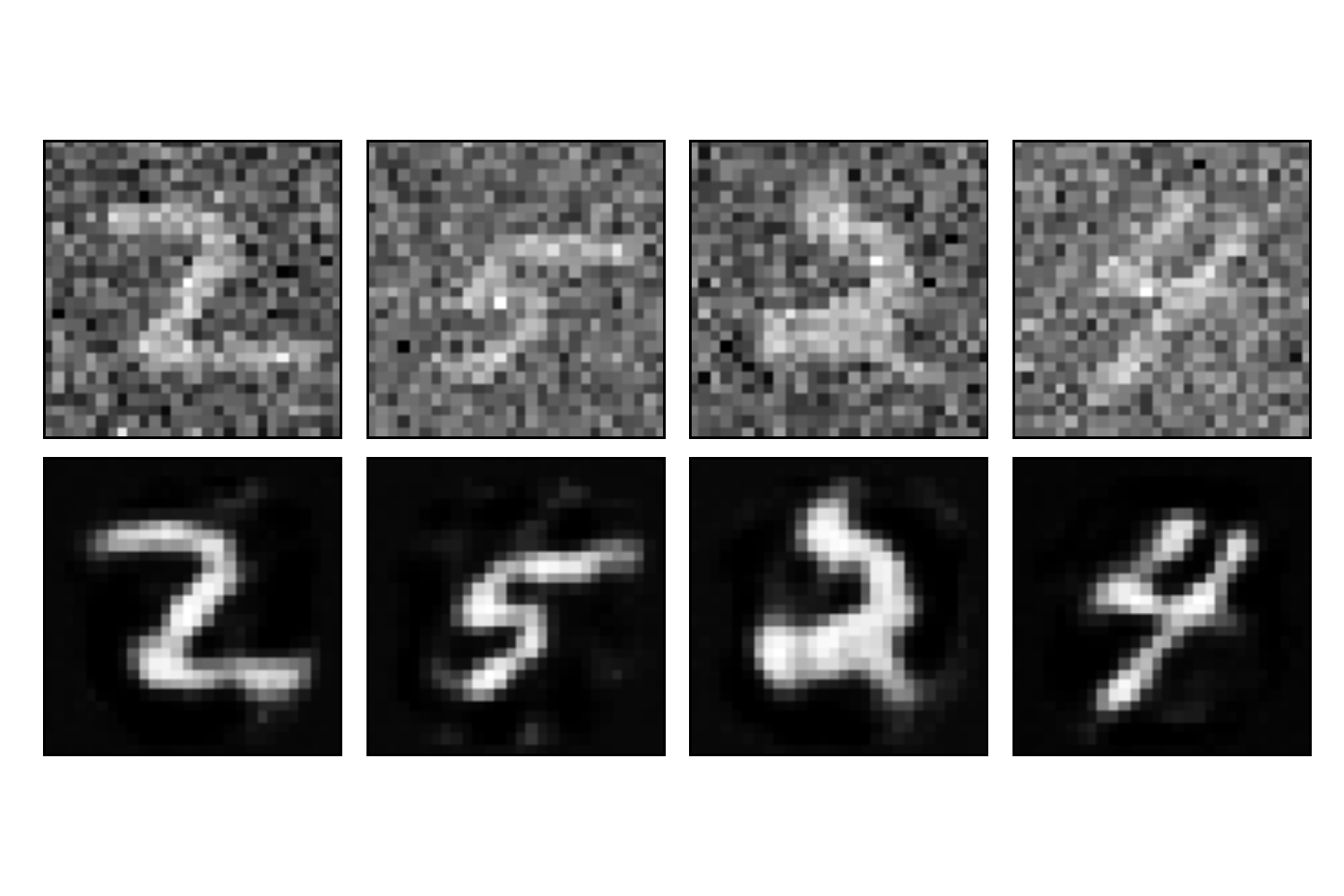}

}\subfloat[$L=6$]{\includegraphics[viewport=60bp 20bp 350bp 290bp,clip,scale=0.4]{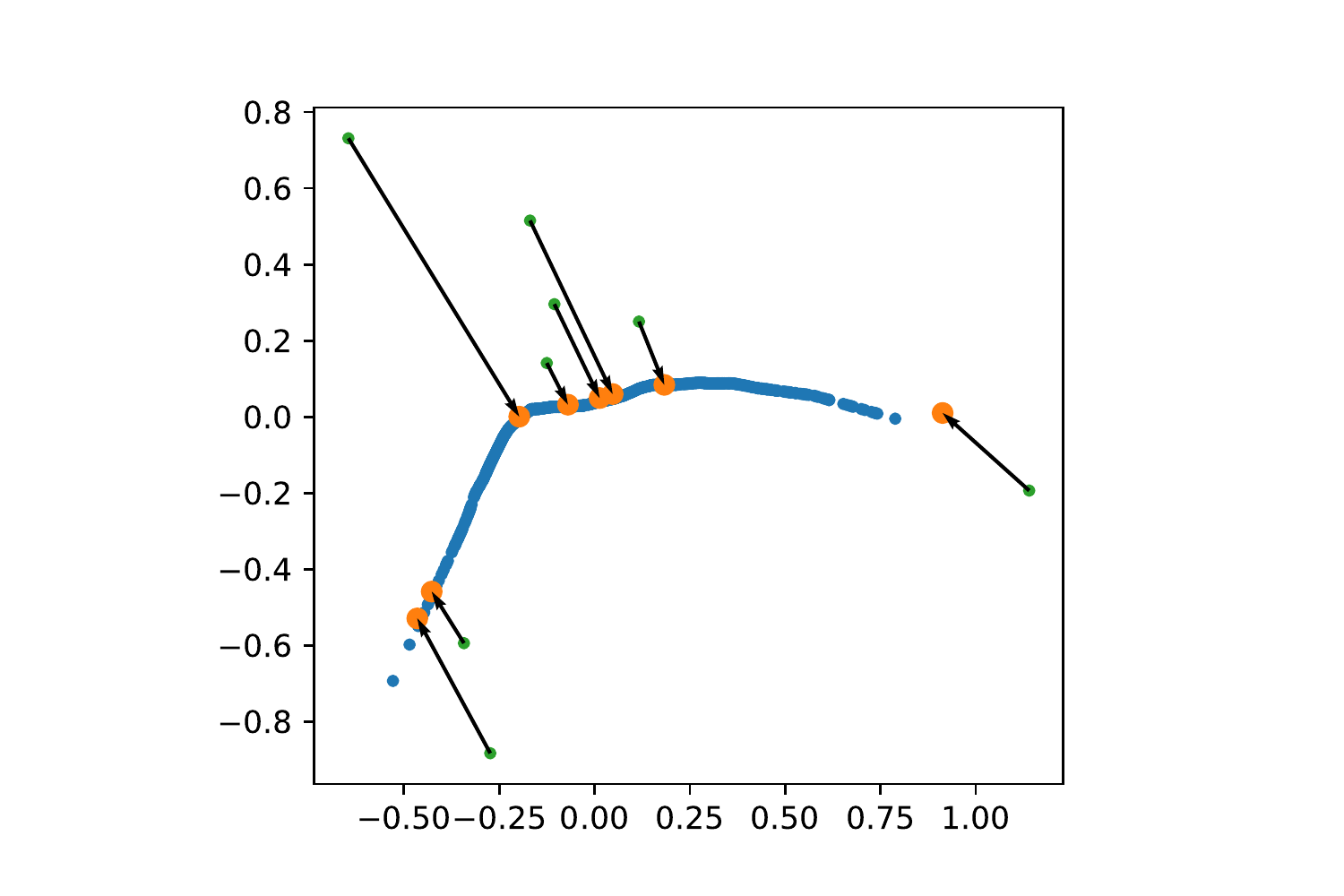}

}\subfloat[$L=2$]{\includegraphics[viewport=80bp 20bp 350bp 290bp,clip,scale=0.4]{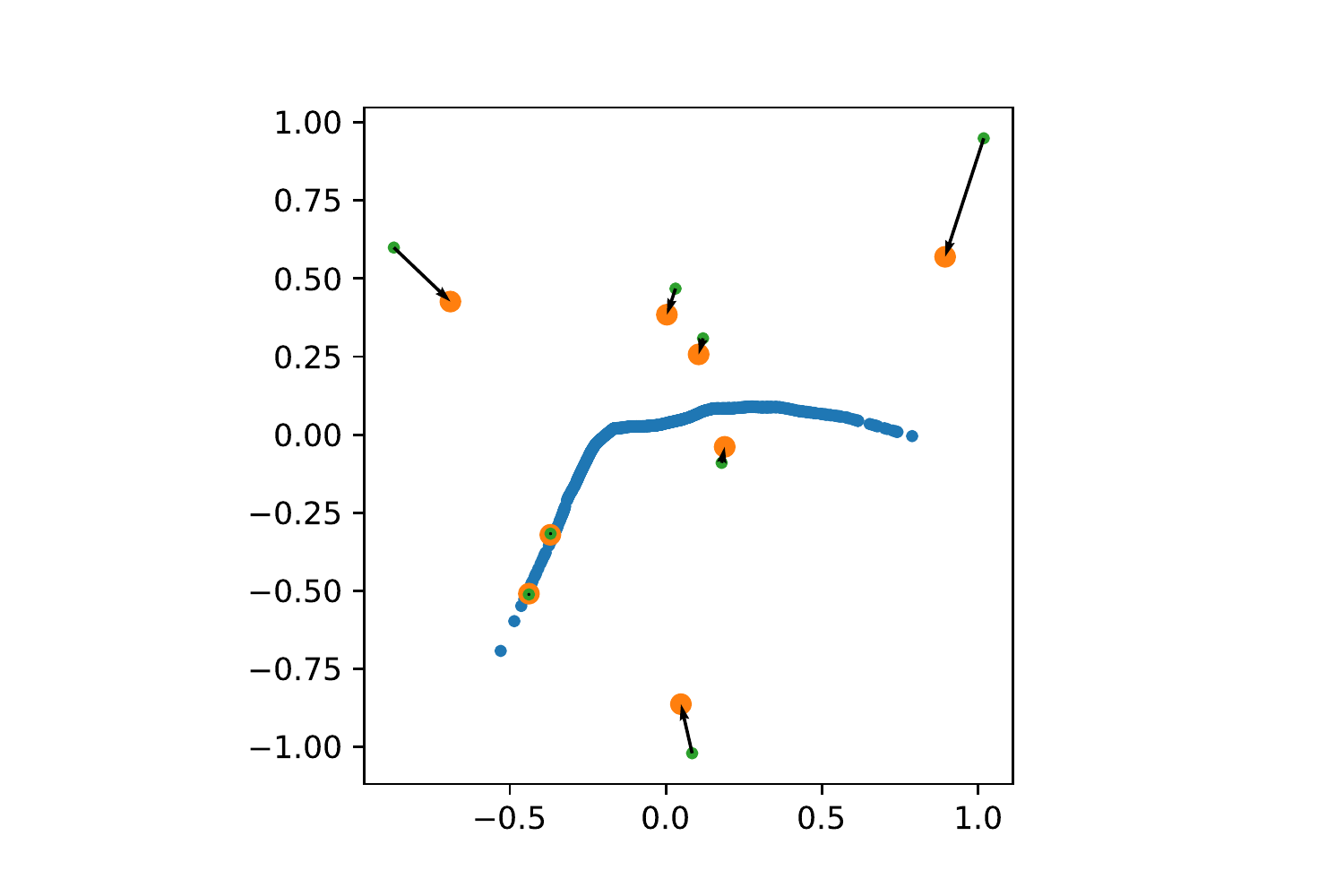}

}

\end{center}

\caption{Autoencoders trained on MNIST \textbf{(A)} and a 1D dataset on the
plane \textbf{(B, C)} with a ridge $\lambda=10^{-4}$. Plot \textbf{(A)}
shows noisy inputs in the first line with corresponding outputs below.
In plots \textbf{(B)} and \textbf{(C)} the blue dots are the training
data, and the green dots are random inputs that are mapped to the
orange dots pointed by the arrows. We see that for large depths \textbf{(A,
B)} the learned autoencoder is naturally denoising, projecting points
to the data distribution, which is not the case for shallow networks
\textbf{(C)}.}
\end{figure}

\subsection{Autoencoders}

Consider learning an autoender on data of the form $x=g(z)$ where
$z$ is sampled (with full dimensional support) in a latent space
$\mathbb{R}^{k}$ and $g:\mathbb{R}^{k}\to\mathbb{R}^{d}$ is an injective
FPLF. In this setting, the true rank is the intrinsic dimension $k$
of the data, since the minimal rank function that equals the identity
on the data distribution has rank $k$.

Assume that the learned autoencoder $\hat{f}:\mathbb{R}^{k}\to\mathbb{R}^{k}$
fits the data $f(x)=x$ for all $x=g(z)$ and recovers the rank $\mathrm{Rank}_{BN}\hat{f}=k$.
At any datapoint $x_{0}=g(z_{0})$ such that $g$ is differentiale
at $z_{0}$, the data support $g(\mathbb{R}^{k})$ is locally a $k$-dimensional
affine subspace $T=x_{0}+\mathrm{Im}Jg(z_{0})$. In the linear region
of $\hat{f}$ that contains $x_{0}$, $\hat{f}$ is an affine projection
to $T$ since it equals the identity when restricted to $T$ and its
Jacobian is rank $k$. This proves that rank recovering autoencoders
are naturally (locally) denoising.

\section{Conclusion}

We have shown that in infinitely deep networks, $L_{2}$-regularization
leads to a bias towards low-rank functions, for some notion of rank
on FPLFs. We have then shown a set of results that suggest that this
low-rank bias extends to large but finite depths. With the right depths,
this leads to `rank recovery', where the learned function has approximately
the same rank as the `true function'. We proposed a tentative explanation
for this rank recovery: for finite but large widths, the network is
biased towards function $f$ which can be factorized $f=h\circ g$
with both a small inner dimension $k$ and small complexity of $g$
and $h$. Finally, we have shown how rank recovery affects the topology
of the class boundaries in a classification task and leads to natural
denoising abilities in autoencoders.

\bibliographystyle{iclr2023_conference}
\bibliography{./../../main}

\appendix

\section{Notions of Rank}
\begin{claim}
From properties (2),(4) follows:
\begin{enumerate}
\item For any function $f$, one has $\mathrm{Rank}f\leq\min\{d_{in},d_{out}\}$.
\item For any bijection $\phi$ on $\mathbb{R}^{d}$, $\mathrm{Rank}\phi=\mathrm{Rank}\phi^{-1}=d$.
\item For any two bijections $\phi,\psi$ on $\mathbb{R}^{d_{in}}$ and
$\mathbb{R}^{d_{out}}$ resp. one has $\mathrm{Rank}\left(\psi\circ f\circ\phi\right)=\mathrm{Rank}f$.
\end{enumerate}
\end{claim}
\begin{proof}
1. By property 4, one has that $\mathrm{Rank}id=d$ for the identity
$id:\mathbb{R}^{d}\to\mathbb{R}^{d}$. By property (4), one has $\mathrm{Rank}f=\mathrm{Rank}(id\circ f\circ id)\leq\min\left\{ d_{in},\mathrm{Rank}f,d_{out}\right\} \leq\min\left\{ d_{in},d_{out}\right\} $.

2. We have $d=\mathrm{Rank}\left(\phi\circ\phi^{-1}\right)\leq\min\left\{ \mathrm{Rank}\phi,\mathrm{Rank}\phi^{-1}\right\} $
and $\mathrm{Rank}\phi\leq d$ as well as $\mathrm{Rank}\phi^{-1}\leq d$.
Therefore $\mathrm{Rank}\phi=\mathrm{Rank}\phi^{-1}=d$.

3. Let us only show $\mathrm{Rank}\left(f\circ\phi\right)=\mathrm{Rank}f$,
the other side follows from the same argument. We have $\mathrm{Rank}\left(f\circ\phi\right)\leq\min\left\{ \mathrm{Rank}f,d_{in}\right\} =\mathrm{Rank}f$
and $\mathrm{Rank}f=\mathrm{Rank}\left(f\circ\phi\circ\phi^{-1}\right)\leq\min\left\{ \mathrm{Rank}\left(f\circ\phi\right),d_{in}\right\} =\mathrm{Rank}\left(f\circ\phi\right)$,
thus proving $\mathrm{Rank}\left(f\circ\phi\right)=\mathrm{Rank}f$.
\end{proof}
\begin{prop}[Proposition 1 in the main]
We have 
\[
\mathrm{Rank}_{J}(f;\Omega)\leq\mathrm{Rank}_{BN}(f;\Omega).
\]
\end{prop}
\begin{proof}
Since $f=g\circ h$ with an inner dimension of $\mathrm{Rank}_{BN}(f;\Omega)$
then at any point $x$ where $f$ is differentiable, we have by the
chain rule
\[
Jf(x)=Jg(h(x))Jh(x).
\]
Clearly the rank of $Jf(x)$ is bounded by the inner dimension $\mathrm{Rank}_{BN}(f;\Omega)$.
\end{proof}
Let us now give an example of a function $f$ where the above inequality
is strict:
\begin{example}
Consider the piecewise linear function $f:\mathbb{R}^{2}\to\mathbb{R}^{2}$
which maps $x=(x_{0},x_{1})$ to $(x_{0},\mathrm{sign}(x_{1})\left|x_{0}\right|)$
if $\left|x_{0}\right|\geq\left|x_{1}\right|$ and to $(\mathrm{sign}(x_{0})\left|x_{1}\right|,x_{1})$
if $\left|x_{0}\right|<\left|x_{1}\right|$. 
\end{example}
\begin{proof}
One can easily check that this function is continuous and equals the
identity on the $x$-cross $X=\left\{ (x_{0},x_{1}):\left|x_{0}\right|=\left|x_{1}\right|\right\} $.
Inside the linear regions (i.e. outside of the $x$-cross and the
$+$-cross made up of the union of both axis) the Jacobian is rank
1, as a result the function $f$ satisfies $\mathrm{Rank}_{J}(f;\mathbb{R}^{2})=1$,
on the other hand $\mathrm{Rank}_{BN}(f;\mathbb{R}^{2})>1$ since
$\mathrm{Rank}_{BN}(f;\mathbb{R}^{2})\geq\mathrm{Rank}_{BN}\left(f;X\right)$
and since there are no continuous functions $g:X\to\mathbb{R}$ and
$h:\mathbb{R}\to X$ such that $h\circ g=id_{X}$ we know that $\mathrm{Rank}_{BN}\left(f;X\right)>1$.
We therefore know that $\mathrm{Rank}_{BN}(f;\mathbb{R}^{2})=2$ (since
$1<\mathrm{Rank}_{BN}(f;\mathbb{R}^{2})\leq2$).
\end{proof}
Finally, one can easily check that both $\mathrm{Rank}_{J}$ and $\mathrm{Rank}_{BN}$
satisfy properties 1-4.

\section{Representation Cost}
\begin{prop}[Proposition 3 in the main]
\label{prop:appendix_bound_repr_norm_Jacobian_schatten}Let $f$
be a piecewise linear function, then at any differentiable point $x$,
we have
\[
\left\Vert Jf(x)\right\Vert _{\nicefrac{2}{L}}^{\nicefrac{2}{L}}:=\sum_{k=1}^{\mathrm{Rank}Jf_{\mathbf{W}}(x)}s_{k}\left(Jf(x)\right)^{\frac{2}{L}}\leq\frac{1}{L}R(f;\Omega,\sigma_{a},L),
\]
where $s_{k}\left(Jf_{\mathbf{W}}(x)\right)$ is the $k$-th singular
value of the Jacobian $Jf_{\mathbf{W}}(x)$.
\end{prop}
\begin{proof}
For any weights $\mathbf{W}$ of a depth $L$ network such that $f_{\mathbf{W}}=f$
we have
\[
Jf(x)=W_{L}D_{L-1}(x)W_{L-1}\cdots W_{2}D_{1}(x)W_{1}
\]
where $D_{\ell}(x)$ is a $n_{\ell}\times n_{\ell}$ diagonal matrix
with diagonal vector equal to $\dot{\sigma_{a}}\left(\tilde{\alpha}_{\ell}(x)\right)$.

We know from \cite{soudry2018implicit} that the representation cost
of linear fully connected networks equals $L\left\Vert A\right\Vert _{p}^{p}$
for $\left\Vert A\right\Vert _{p}^{p}=\lambda_{1}^{p}+\dots+\lambda_{k}^{p}$
is the $L_{p}$-Schatten norm with $p=\frac{2}{L}$. In other terms,
we have for any matrices $\tilde{W}_{1},\dots,\tilde{W}_{L}$
\[
L\left\Vert \tilde{W}_{L}\cdots\tilde{W}_{1}\right\Vert _{p}^{p}\leq\left\Vert \tilde{W}_{L}\right\Vert _{F}^{2}+\dots+\left\Vert \tilde{W}_{1}\right\Vert _{F}^{2}.
\]
Applying it to $\tilde{W}_{L}=W_{L}$ and $\tilde{W}_{\ell}=D_{\ell}(x)W_{\ell}$
for $\ell=1,\dots,L-1$, we obtain
\begin{align*}
\left\Vert Jf(x)\right\Vert _{p}^{p} & \leq\frac{\left\Vert W_{L}\right\Vert _{F}^{2}+\left\Vert D_{L-1}(x)W_{L-1}\right\Vert _{F}^{2}+\dots+\left\Vert D_{1}(x)W_{1}\right\Vert _{F}^{2}}{L}\\
 & \leq\frac{\left\Vert W_{L}\right\Vert _{F}^{2}+\left\Vert W_{L-1}\right\Vert _{F}^{2}+\dots+\left\Vert W_{1}\right\Vert _{F}^{2}}{L}
\end{align*}
since $\left\Vert D_{\ell}(x)\right\Vert _{op}\leq1$.

Note that this result applies for any widths $n_{1},\dots,n_{L-1}$.
\end{proof}
\begin{thm}[first part of Theorem 1 in the main]
\label{thm:sandwich-bound-rank-representation-cost}We have
\[
\mathrm{Rank}_{J}(f;\Omega)\leq\lim_{L\to\infty}\frac{R(f;\Omega,\sigma_{a},L)}{L}\leq\mathrm{Rank}_{BN}(f;\Omega).
\]
\end{thm}
\begin{proof}
\textbf{First inequality:} Take a point $x$ such that $\mathrm{Rank}(Jf(x))=\mathrm{Rank}_{J}(f;\Omega)$,
then Proposition \ref{prop:appendix_bound_repr_norm_Jacobian_schatten}
implies that $\frac{R(f;\Omega,\sigma_{a},L)}{L}\geq\left\Vert Jf(x)\right\Vert _{p}^{p}$.
Letting $L\to\infty$ on both sides leads to the bound $\lim_{L\to\infty}\frac{R(f;\Omega,\sigma_{a},L)}{L}\geq\mathrm{Rank}(Jf(x))=\mathrm{Rank}_{J}(f;\Omega)$
as needed. 

This lower bound applies to any widths $n_{1}(L),\dots,n_{L-1}(L)$,
of course if the widths are too small, it might be impossible to represent
$f$, in which case $R(f;\Omega,\sigma_{a},L)=\infty$.

\textbf{Second Inequality:} Fix a decomposition $f=g\circ h$ with
minimal inner dimension and such that $h(\Omega)\subset\mathbb{R}_{+}^{\mathrm{Rank}(f;\Omega)}$
(we need to be in the upper quadrant to represent the identity on
$h(\Omega)$ efficiently, and since $\Omega$ is bounded, one can
always translate the output of $h$ to be in the upper quadrant).

Corollary \ref{thm:piecewise-linear-func-representation} tells us
that there are two networks of finite depths $L_{h}$ and $L_{g}$
(with parameters $\mathbf{W}_{h}$ and $\mathbf{W}_{g}$) which represent
$h$ and $g$, for any depth $L$ larger than $L_{h}+L_{g}$ we can
construct a network of depth $L$ which represents $f$ by concatenating
the network that the represents $h$, followed by $L-L_{h}-L_{g}$
identity weight matrices of dimension $\mathrm{Rank}(f;\Omega)\times\mathrm{Rank}(f;\Omega)$
and finally the network representing $g$. The norm of the parameters
of this network is $\left\Vert \mathbf{W}_{h}\right\Vert ^{2}+(L-L_{h}-L_{g})\mathrm{Rank}(f;\Omega)+\left\Vert \mathbf{W}_{g}\right\Vert ^{2}$.
We therefore have the bound
\[
R(f;\Omega,\sigma_{a},L)\leq\left\Vert \mathbf{W}_{h}\right\Vert ^{2}+(L-L_{h}-L_{g})\mathrm{Rank}_{BN}(f;\Omega)+\left\Vert \mathbf{W}_{g}\right\Vert ^{2}
\]
 divinding both sides by $L$ and letting $L$ grow to infinity, we
obtain the inequality $\lim_{L\to\infty}\frac{R(f;\Omega,\sigma_{a},L)}{L}\leq\mathrm{Rank}_{BN}(f;\Omega)$.

For the upper bound to apply, the widths $n_{\ell}$ of the network
in the first part must be larger than some threshold that depends
on the number of linear regions in $h$, in the middle part the widths
must be larger than $k$ and in the last part they must be above a
threshold that depends on the number of linear regions in $g$ \cite{he_2018_relu_piecewise_lin}.
Note that in each of these regions the minimal with required does
not depend on the depth.
\end{proof}
Let us now show that the limiting rescaled representation cost $R_{\infty}(f;\Omega,\sigma_{a}):=\lim_{L\to\infty}\frac{R(f;\Omega,\sigma_{a},L)}{L}$
satisfies all properties of rank except the first one (though it might
actually satisfy it):
\begin{thm}[second part of Theorem 1 in the main]
 We have for any piecewise linear functions $f,g$:
\begin{enumerate}
\item $R_{\infty}(f\circ g;\Omega,\sigma_{a})\leq\min\{R_{\infty}(f;g(\Omega),\sigma_{a}),R_{\infty}(g;\Omega,\sigma_{a})\}$.
\item $R_{\infty}(f+g;\Omega,\sigma_{a})\leq R_{\infty}(f;\Omega,\sigma_{a})+R_{\infty}(g;\Omega,\sigma_{a})$.
\item If $f$ is affine ($f(x)=Ax+b$) then $R_{\infty}\left(f;\Omega,\sigma_{a}\right)=\mathrm{Rank}A$.
\end{enumerate}
\end{thm}
\begin{proof}
\textbf{1.} Without loss of generality, we can translate the output
of $g$ and the input of $f$ (keeping the same composition $f\circ g$)
so that $g(\Omega)$ lies in the upper quadrant $\mathbb{R}_{+}^{m}$
where $m$ is the inner dimension. This translation changes the parameter
norm by a value which is constant in $L$, it therefore does not matter
in the $L\to\infty$ limit of $\nicefrac{\left\Vert \mathbf{W}\right\Vert ^{2}}{L}$. 

Assume $R_{\infty}(f;\Omega,\sigma_{a})\leq R_{\infty}(g;\Omega,\sigma_{a})$
(the other case can be proved with the same argument) and fix a network
of depth $L_{0}$ and parameters $\mathbf{W}_{0}$ that represents
the function $g$. For any $L$ sufficiently large we consider the
network made up of the composition of the fixed network followed by
a network of depth $L-L_{0}$ with weigths $\mathbf{W}'$ which represents
$f$ with minimal parameter norm, i.e. $\left\Vert \mathbf{W}'\right\Vert ^{2}=R(f;g(\Omega),\sigma_{a},L-L_{0})$.
The norm of this composed network is $\left\Vert \mathbf{W}_{0}\right\Vert ^{2}+R(f;g(\Omega),\sigma_{a},L-L_{0})$,
in the $L\to\infty$ limit, this implies $R_{\infty}(f\circ g;\Omega,\sigma_{a})\leq R_{\infty}(f;g(\Omega),\sigma_{a})$
as needed.

\textbf{2.} For any sufficiently large depth $L$ consider two networks
of depth $L$ with parameters $\mathbf{W}_{f}$ and $\mathbf{W}_{g}$
which represent the functions $f$ and $g$ with minimal parameter
norms, i.e. $\left\Vert \mathbf{W}_{f}\right\Vert ^{2}=R(f;\Omega,\sigma_{a},L)$
and $\left\Vert \mathbf{W}_{g}\right\Vert ^{2}=R(g;\Omega,\sigma,L)$.
We then consider the network obtained by putting the two network in
'parallel', i.e. the first weight matrix is given by the concatenation
$\left(\begin{array}{c}
W_{f,1}\\
W_{g,1}
\end{array}\right)$, the weight matrices of the middle layers are of the form $\left(\begin{array}{cc}
W_{f,\ell} & 0\\
0 & W_{g,\ell}
\end{array}\right)$ for all $\ell=2,\dots,L-1$ and the last weight matrix is given by
$\left(\begin{array}{cc}
W_{f,L} & W_{g,L}\end{array}\right)$. This new network represents the function $f+g$ and has parameter
norm $\left\Vert \mathbf{W}_{f}\right\Vert ^{2}+\left\Vert \mathbf{W}_{g}\right\Vert ^{2}$,
which implies the bound $R(f+g;\Omega,\sigma_{a},L)\leq\left\Vert \mathbf{W}_{f}\right\Vert ^{2}+\left\Vert \mathbf{W}_{g}\right\Vert ^{2}=R(f;\Omega,\sigma_{a},L)+R(g;\Omega,\sigma_{a},L)$
and in the limit $R_{\infty}(f+g;\Omega,\sigma_{a})\leq R_{\infty}(f;\Omega,\sigma_{a})+R_{\infty}(g;\Omega,\sigma_{a})$.

\textbf{3.} This point follows from Theorem \ref{thm:sandwich-bound-rank-representation-cost},
since for affine functions both notions of rank agree $\mathrm{Rank}_{J}(f;\Omega)=\mathrm{Rank}_{BN}(f;\Omega)=\mathrm{Rank}A$,
the same must be true for the limiting rescaled representation cost
$R_{\infty}(f;\Omega,\sigma_{a})$.

Regarding the widths required for these results to apply, the minimal
widths are the one described by the construction in the proofs. Since
these constructions can be included into any wider network by adding
zero neurons (neurons with zero incoming weights and zero outcoming
weights), these results also apply to any larger widths.
\end{proof}

\section{Global Minima are Almost Rank 1}

Consider a global minimizer $\hat{\mathbf{W}}$ of the regression
problem $\mathcal{L}(\mathbf{W})=\frac{1}{N}\sum\left(f_{\mathbf{W}}(x_{i})-y_{i}\right)^{2}+\frac{\lambda}{L}\left\Vert \mathbf{W}\right\Vert ^{2}$
we will now show that if the depth $L$ is large enough, then the
function $f_{\hat{\mathbf{W}}}$ is in a sense almost rank 1 w.r.t.
both notions of rank ($\mathrm{Rank}_{J}$ and $\mathrm{Rank}_{BN}$
).
\begin{prop}[Proposition 2 in the main]
For the regression problem $\mathcal{L}_{\lambda}(\mathbf{W})=\frac{1}{N}\sum\left(f_{\mathbf{W}}(x_{i})-y_{i}\right)^{2}+\frac{\lambda}{L}\left\Vert \mathbf{W}\right\Vert ^{2}$
there is a constant $C_{N}$ (which depends only on the inputs $x_{i}$
and outputs $y_{i}$) such that for $L\geq\left\lceil \log_{2}(n_{0}+1)\right\rceil +2$,
we have\textbf{
\[
\inf_{\mathbf{W}}\mathcal{L}_{\lambda}(\mathbf{W})\leq\lambda\left(1+\frac{C_{N}}{L}\right).
\]
}
\end{prop}
\begin{proof}
There is a BN-rank $1$ function $f=h\circ g$ (with $g:\mathbb{R}^{n_{0}}\to\mathbb{R}$
and $h:\mathbb{R}\to\mathbb{R}^{n_{L}}$) which fits the data perfectly
$f(x_{i})=y_{i}$ for all $i$. Much like in the second inequality
of Theorem \ref{thm:sandwich-bound-rank-representation-cost}, by
Corollary \ref{thm:piecewise-linear-func-representation}, there is
a depth $\left\lceil \log_{2}(n_{0}+1)\right\rceil $ networks which
represents $g$ and a depth $2$ network which represents the function
$f$. For any depth $L>\left\lceil \log_{2}(n_{0}+1)\right\rceil +2$,
we compose the network representing $g$, followed by a number of
identity layers, followed by the network representing $h$. The function
represented by this network has zero loss and the regularization term
is of the form $\lambda\left(1+\frac{C_{N}}{L}\right)$, since
\[
\frac{1}{L}\left\Vert \mathbf{W}\right\Vert ^{2}=\frac{\left\Vert \mathbf{W}_{g}\right\Vert ^{2}}{L}+\left(1-\frac{\left\lceil \log_{2}(n_{0}+1)\right\rceil +2}{L}\right)+\frac{\left\Vert \mathbf{W}_{h}\right\Vert ^{2}}{L}.
\]

Note that the minimal width required for this result to apply might
depend on the number of datapoints $N$, but not the depth.
\end{proof}
Let us now show that the function $f_{\hat{\mathbf{W}}}$ is close
to a BN-rank 1 function:
\begin{prop}[Proposition 4 in the main]
\label{thm:almost_BN_rank_1}For any global minimum $\hat{\mathbf{W}}$
of the $L_{2}$-regularized loss $\mathcal{L}_{\lambda}$ with $\lambda>0$
and any set of $\tilde{N}$ datapoints $\tilde{X}\in\mathbb{R}^{d_{in}\times\tilde{N}}$
(which do not have to be the training set $X$) with non-constant
outputs, there is a layer $\ell_{0}$ such that the first two singular
values $s_{1},s_{2}$ of the hidden representation $Z_{\ell_{0}}\in\mathbb{R}^{n_{\ell}\times N}$
(whose columns are the activations $\alpha_{\ell_{0}}(x_{i})$ for
all the inputs $x_{i}$ in $\tilde{X}$) satisfies $\frac{s_{2}}{s_{1}}=O(L^{-\frac{1}{4}})$.
\end{prop}
\begin{proof}
We need to prove a lower bound on the first eigenvalue of $\frac{1}{\tilde{N}}Z_{\ell}^{T}Z_{\ell}$
and an upper bound on the second one. For both parts, we will rely
on the balanced property described in Proposition \ref{prop:balancedness}:
at any local minimum of the loss the weights satisfy $\left\Vert W_{\ell+1}\right\Vert _{F}^{2}=\left\Vert W_{\ell}\right\Vert _{F}^{2}+\left\Vert b_{\ell}\right\Vert ^{2}$.
This implies that $\left\Vert W_{\ell'}\right\Vert _{F}^{2}\geq\left\Vert W_{\ell}\right\Vert _{F}^{2}$
for all $\ell'\geq\ell$. Since the overall norm of the parameters
is bounded by $L+C_{N}$ this implies a bound
\[
\left\Vert W_{\ell+1}\right\Vert _{F}^{2}\leq\frac{L+C_{N}}{L-\ell}=1+\frac{C_{N}+\ell}{L-\ell}
\]
for all $\ell$ .

Assuming by contradiction that for all layers $\frac{\lambda_{2}\left(\frac{1}{\tilde{N}}Z_{\ell}^{T}Z_{\ell}\right)}{\lambda_{1}\left(\frac{1}{\tilde{N}}Z_{\ell}^{T}Z_{\ell}\right)}>\delta$,
one should intuitively think of $1+\frac{C_{N}+\ell}{L-\ell}$ as
a `ressource' with which the $\ell$-th layer has to do two tasks:
(1) keep the top eigenvalue of $\frac{1}{\tilde{N}}Z_{\ell}^{T}Z_{\ell}$
close to 1 to keep enough information to represent the outputs; and
(2) keep the second eigenvalue above $\delta$ to keep the contradiction.
However the ressource cost of (1) is roughly $1$ and the cost of
(2) is roughly $\delta$ which is above the ressource allowance $1+\frac{C_{N}+\ell}{L-\ell}$
for large $L$ and constant $\ell$. This leads to a contradiction.

\textbf{Upper bound on $\lambda_{2}$:} Let $\ell_{0}$ be the first
time where $\lambda_{2}<\delta$, we will show that $\ell_{0}$ exists
and is upper bounded by $\frac{2\left\Vert \frac{1}{\tilde{N}}\tilde{X}^{T}\tilde{X}\right\Vert _{op}}{\delta}$.

For all $\ell<\ell_{0}$ we have the following for bound the operator
norm $\left\Vert \frac{1}{\tilde{N}}Z_{\ell}^{T}Z_{\ell}\right\Vert _{op}$:
\begin{align*}
\left\Vert \frac{1}{\tilde{N}}Z_{\ell}^{T}Z_{\ell}\right\Vert _{op} & \leq\mathrm{Tr}\left[\frac{1}{\tilde{N}}Z_{\ell}^{T}Z_{\ell}\right]-\lambda_{2}\left(\frac{1}{\tilde{N}}Z_{\ell}^{T}Z_{\ell}\right)\\
 & \leq\mathrm{Tr}\left[\frac{1}{\tilde{N}}\sigma_{a}\left(W_{\ell}Z_{\ell-1}+b_{\ell}\right)^{T}\sigma_{a}\left(W_{\ell}Z_{\ell-1}+b_{\ell}\right)\right]-\delta\\
 & \leq\left\Vert \frac{1}{\tilde{N}}Z_{\ell-1}^{T}Z_{\ell-1}\right\Vert _{op}\left\Vert W_{\ell}\right\Vert _{F}^{2}+\left\Vert b_{\ell}\right\Vert _{F}^{2}-\delta.
\end{align*}
 Assuming $\left\Vert \frac{1}{\tilde{N}}Z_{\ell}^{T}Z_{\ell}\right\Vert _{op}\leq\max\left\{ \left\Vert \frac{1}{\tilde{N}}\tilde{X}^{T}\tilde{X}\right\Vert _{op},1\right\} $
(we will later show that this is true for all $\ell\leq\ell_{0}$
as long as $L$ is sufficiently large) we obtain:

\begin{align}
\left\Vert \frac{1}{\tilde{N}}Z_{\ell}^{T}Z_{\ell}\right\Vert _{op} & \leq\left\Vert \frac{1}{\tilde{N}}Z_{\ell-1}^{T}Z_{\ell-1}\right\Vert _{op}+\max\left\{ \left\Vert \frac{1}{\tilde{N}}\tilde{X}^{T}\tilde{X}\right\Vert _{op},1\right\} \left(\left\Vert W_{\ell}\right\Vert _{F}^{2}+\left\Vert b_{\ell}\right\Vert _{F}^{2}-1\right)-\delta.\label{eq:rec_op_norm_representation}
\end{align}
Since $\left\Vert W_{\ell}\right\Vert _{F}^{2}+\left\Vert b_{\ell}\right\Vert _{F}^{2}=\left\Vert W_{\ell+1}\right\Vert _{F}^{2}\leq1+\frac{\ell+C_{N}}{L-\ell}$,
if $L>\frac{2\max\left\{ \left\Vert \frac{1}{\tilde{N}}\tilde{X}^{T}\tilde{X}\right\Vert _{op},1\right\} \left(C_{N}+\left\lceil \frac{2\left\Vert \frac{1}{\tilde{N}}\tilde{X}^{T}\tilde{X}\right\Vert _{op}}{\delta}\right\rceil \right)}{\delta}+\left\lceil \frac{2\left\Vert \frac{1}{\tilde{N}}\tilde{X}^{T}\tilde{X}\right\Vert _{op}}{\delta}\right\rceil \geq\frac{\kappa}{\delta^{2}}$
for some $\kappa$ (which depends on $\left\Vert \frac{1}{\tilde{N}}\tilde{X}^{T}\tilde{X}\right\Vert _{op}$
and $C_{N}$ only) and $L$ large enough, then for all $\ell\leq\min\left\{ \left\lceil \frac{2\left\Vert \frac{1}{\tilde{N}}\tilde{X}^{T}\tilde{X}\right\Vert _{op}}{\delta}\right\rceil ,\ell_{0}\right\} $,
we obtain that
\begin{align*}
\left\Vert \frac{1}{\tilde{N}}Z_{\ell}^{T}Z_{\ell}\right\Vert _{op} & \leq\left\Vert \frac{1}{\tilde{N}}Z_{\ell-1}^{T}Z_{\ell-1}\right\Vert _{op}-\frac{\delta}{2}\\
 & \leq\left\Vert \frac{1}{\tilde{N}}\tilde{X}^{T}\tilde{X}\right\Vert _{op}-\ell\frac{\delta}{2}.
\end{align*}

Therefore for all $\ell\leq\min\left\{ \left\lceil \frac{2\left\Vert \frac{1}{\tilde{N}}\tilde{X}^{T}\tilde{X}\right\Vert _{op}}{\delta}\right\rceil ,\ell_{0}\right\} $
we have $\left\Vert \frac{1}{\tilde{N}}Z_{\ell}^{T}Z_{\ell}\right\Vert _{op}\leq\max\left\{ \left\Vert \frac{1}{\tilde{N}}\tilde{X}^{T}\tilde{X}\right\Vert _{op},1\right\} $,
as needed. Furthermore this implies that $\ell_{0}\leq\left\lceil \frac{2\left\Vert \frac{1}{\tilde{N}}\tilde{X}^{T}\tilde{X}\right\Vert _{op}}{\delta}\right\rceil $,
otherwise, we would get a contradiction when taking $\ell=\left\lceil \frac{2\left\Vert \frac{1}{\tilde{N}}\tilde{X}^{T}\tilde{X}\right\Vert _{op}}{\delta}\right\rceil $:
\[
\left\Vert \frac{1}{\tilde{N}}Z_{\ell}^{T}Z_{\ell}\right\Vert _{op}\leq\left\Vert \frac{1}{\tilde{N}}\tilde{X}^{T}\tilde{X}\right\Vert _{op}-\left\lceil \frac{2\left\Vert \frac{1}{\tilde{N}}\tilde{X}^{T}\tilde{X}\right\Vert _{op}}{\delta}\right\rceil \frac{\delta}{2}<0.
\]

We have now proven that for large enough $L$, there is a $\kappa$
(which depends on $\left\Vert \frac{1}{\tilde{N}}\tilde{X}^{T}\tilde{X}\right\Vert _{op}$
and $C_{N}$ only) such there is a $\ell_{0}\leq\left\lceil \frac{2\sqrt{L}\left\Vert \frac{1}{\tilde{N}}\tilde{X}^{T}\tilde{X}\right\Vert _{op}}{\delta\sqrt{\kappa}}\right\rceil $
where $\lambda_{2}\left(\frac{1}{\tilde{N}}Z_{\ell_{0}}^{T}Z_{\ell_{0}}\right)<\sqrt{\frac{\kappa}{L}}$.

\textbf{Lower bound on $\lambda_{1}$:} We now need to lower bound
the first eigenvalue $\lambda_{1}\left(\frac{1}{\tilde{N}}Z_{\ell_{0}}^{T}Z_{\ell_{0}}\right)=\left\Vert \frac{1}{\tilde{N}}Z_{\ell_{0}}^{T}Z_{\ell_{0}}\right\Vert _{op}$
at this same layer $\ell_{0}$ . We denote the means $m_{\ell}=\frac{1}{\tilde{N}}\sum_{i=1}^{\tilde{N}}\alpha_{\ell}(x_{i})$
and have the bounds
\[
\left\Vert \frac{1}{\tilde{N}}\left(Z_{\ell}-m_{\ell}\right)^{T}\left(Z_{\ell}-m_{\ell}\right)\right\Vert _{op}\leq\left\Vert \frac{1}{\tilde{N}}\left(Z_{\ell-1}-m_{\ell-1}\right)^{T}\left(Z_{\ell-1}-m_{\ell-1}\right)\right\Vert _{op}\left\Vert W_{\ell}\right\Vert _{op}^{2}.
\]
This implies that 
\begin{equation}
\left\Vert \frac{1}{\tilde{N}}Z_{\ell_{0}}^{T}Z_{\ell_{0}}\right\Vert _{op}\geq\left\Vert \frac{1}{\tilde{N}}\left(Z_{\ell}-m_{\ell_{0}}\right)^{T}\left(Z_{\ell}-m_{\ell_{0}}\right)\right\Vert _{op}\geq\frac{\left\Vert \frac{1}{\tilde{N}}\left(Z_{\ell}-m_{\ell_{0}}\right)^{T}\left(Z_{\ell}-m_{\ell_{0}}\right)\right\Vert _{op}}{\left\Vert W_{\ell_{0}+1}\right\Vert _{op}^{2}\cdots\left\Vert W_{L}\right\Vert _{op}^{2}}.\label{eq:lower_bound_norm_ell0}
\end{equation}

We now need to lower bound the norm of the parameters in the layers
up to $\ell_{0}$, to upper bound the norm of the parameters of the
layers $\ell_{0}+1$ to $L$. Iterating Equation (\ref{eq:rec_op_norm_representation})
leads to the equation
\[
\left\Vert \frac{1}{\tilde{N}}Z_{\ell_{0}}^{T}Z_{\ell_{0}}\right\Vert _{op}\leq\left\Vert \frac{1}{\tilde{N}}\tilde{X}^{T}\tilde{X}\right\Vert _{op}+\max\left\{ \left\Vert \frac{1}{\tilde{N}}\tilde{X}^{T}\tilde{X}\right\Vert _{op},1\right\} \left(\sum_{\ell=1}^{\ell_{0}}\left\Vert W_{\ell}\right\Vert _{F}^{2}+\left\Vert b_{\ell}\right\Vert _{F}^{2}-1\right)-\ell_{0}\delta.
\]
which implies that

\[
\sum_{\ell=1}^{\ell_{0}}\left\Vert W_{\ell}\right\Vert _{F}^{2}+\left\Vert b_{\ell}\right\Vert _{F}^{2}\geq\ell_{0}+\frac{\ell_{0}\delta-\left\Vert \frac{1}{\tilde{N}}\tilde{X}^{T}\tilde{X}\right\Vert _{op}}{\max\left\{ \left\Vert \frac{1}{\tilde{N}}\tilde{X}^{T}\tilde{X}\right\Vert _{op},1\right\} }\geq\ell_{0}-\min\left\{ 1,\left\Vert \frac{1}{\tilde{N}}\tilde{X}^{T}\tilde{X}\right\Vert _{op}^{-1}\right\} 
\]
and therefore
\[
\sum_{\ell=\ell_{0}+1}^{L}\left\Vert W_{\ell}\right\Vert _{F}^{2}+\left\Vert b_{\ell}\right\Vert _{F}^{2}\leq L-\ell_{0}+C_{N}+\min\left\{ 1,\left\Vert \frac{1}{\tilde{N}}\tilde{X}^{T}\tilde{X}\right\Vert _{op}^{-1}\right\} .
\]

Applying the arithmetic/geometric mean inequality to Equation (\ref{eq:lower_bound_norm_ell0}),
we obtain a lower bound
\begin{align*}
\left\Vert \frac{1}{\tilde{N}}Z_{\ell_{0}}^{T}Z_{\ell_{0}}\right\Vert _{op} & \geq\frac{\left\Vert \frac{1}{\tilde{N}}\left(Z_{\ell}-m_{\ell_{0}}\right)^{T}\left(Z_{\ell}-m_{\ell_{0}}\right)\right\Vert _{op}}{\left(\frac{1}{L-\ell_{0}}\sum_{\ell=\ell_{0}+1}^{L}\left\Vert W_{\ell}\right\Vert _{op}^{2}\right)^{L-\ell_{0}}}\\
 & \geq\frac{\left\Vert \frac{1}{\tilde{N}}\left(Z_{\ell}-m_{\ell_{0}}\right)^{T}\left(Z_{\ell}-m_{\ell_{0}}\right)\right\Vert _{op}}{\left(1+\frac{C_{N}+\min\left\{ 1,\left\Vert \frac{1}{\tilde{N}}\tilde{X}^{T}\tilde{X}\right\Vert _{op}^{-1}\right\} }{L-\ell_{0}}\right)^{L-\ell_{0}}}\\
 & \geq e^{-C_{N}-\min\left\{ 1,\left\Vert \frac{1}{\tilde{N}}\tilde{X}^{T}\tilde{X}\right\Vert _{op}^{-1}\right\} }\left\Vert \frac{1}{\tilde{N}}\left(Z_{\ell}-m_{\ell_{0}}\right)^{T}\left(Z_{\ell}-m_{\ell_{0}}\right)\right\Vert _{op}.
\end{align*}

Putting it all together, we have shown that for large enough $L$,
there is a $\ell_{0}=O(\sqrt{L})$ such that $\lambda_{1}\left(\frac{1}{\tilde{N}}Z_{\ell_{0}}^{T}Z_{\ell_{0}}\right)=\Omega(1)$
and $\lambda_{2}\left(\frac{1}{\tilde{N}}Z_{\ell_{0}}^{T}Z_{\ell_{0}}\right)=O\left(\frac{1}{\sqrt{L}}\right)$,
which together imply that
\[
\frac{\lambda_{1}\left(\frac{1}{\tilde{N}}Z_{\ell_{0}}^{T}Z_{\ell_{0}}\right)}{\lambda_{2}\left(\frac{1}{\tilde{N}}Z_{\ell_{0}}^{T}Z_{\ell_{0}}\right)}=O\left(\frac{1}{\sqrt{L}}\right)
\]
and therefore
\[
\frac{s_{1}(Z_{\ell_{0}})}{s_{2}(Z_{\ell_{0}})}=O(L^{-\frac{1}{4}}).
\]

Finally note that this result does not require anything more than
the widths be nonzero. Of course , if one of the widths is 1, the
result is trivial.
\end{proof}

\section{Rank Recovery}

Consider a finite dataset $X,Y$ of size $N$, with $x_{i}$ sampled
i.i.d. for a distribution $p$ with support equal to $\Omega$ and
with $y_{i}=f^{*}(x_{i})$ for a true function $f^{*}:\Omega\to\mathbb{R}^{d_{out}}$
with $\mathrm{Rank}_{J}(f^{*};\Omega)=k>1$. For any function $g$
which fits the data $g(x_{i})=y_{i}$ with a BN-Rank of $1$ (there
always exists at least one such function), then if the depth $L$
is large enough we have $R(g;\Omega,\sigma_{a},L)<R(f^{*};\Omega,\sigma_{a},L)$. 

This is problematic as it suggests that for large depths, minimizing
the representation cost will always lead to fitting the data with
a function with a BN rank of 1 instead of the rank of the true function
$k$. However, if we instead fix a depth $L$ and let the number of
datapoints $N$ grow, the representation cost required to fit the
data with a rank 1 function (or any rank lower than $k$) increases
to infinity, whereas the representation cost of the true function
remains constant. This suggest that if one increases the depth $L$
and the number of datapoints $N$ simultaneously with the right scaling,
minimizing the representation cost over fitting functions should recover
a function $h$ with the right rank $k$.
\begin{prop}[Theorem 2 in the main]
Let $f$ satisfy $f(x_{i})=y_{i}$ and $\mathrm{Rank}_{BN}(f;\Omega)=1$
for some $\Omega$ which contains the convex hull of $x_{1},\dots,x_{N}$.
There is a point $x\in\Omega$, such that 
\[
\left\Vert Jf(x)\right\Vert _{op}\geq\frac{\mathrm{TSP}(y_{1},\dots,y_{N})}{\mathrm{diam}(x_{1},\dots,x_{N})},
\]
for the Traveling Salesman Problem $\mathrm{TSP}(y_{1},\dots,y_{N})$,
i.e. the length of the shortest path passing through every points
$y_{1},\dots,y_{m}$, and for the diameter $\mathrm{diam}(x_{1},\dots,x_{N})$
of the points $x_{1},\dots,x_{N}$. As a result any rank $1$ interpolator
with parameters $\mathbf{W}$ satisfies $\left\Vert \mathbf{W}\right\Vert ^{2}\geq L\left(\frac{\mathrm{TSP}(y_{1},\dots,y_{N})}{\mathrm{diam}(x_{1},\dots,x_{N})}\right)^{\frac{2}{L}}$.
\end{prop}
\begin{proof}
The lower bound on the nom of the paramaters $\left\Vert \mathbf{W}\right\Vert ^{2}$
follows directly from the first bound and Proposition \ref{prop:appendix_bound_repr_norm_Jacobian_schatten}. 

Let us now prove the first bound. Since $\mathrm{Rank}_{BN}(f;\Omega)=1$,
there are piecewise linear functions $g:\mathbb{R}^{d_{in}}\to\mathbb{R}$
and $h:\mathbb{R}\to\mathbb{R}^{d_{out}}$ such that $f=h\circ g$.
We define $z_{i}=g(x_{i})$ and w.l.o.g. we assume that $z_{1}\leq\dots\leq z_{N}$.

The image of the segment $[x_{1},x_{N}]$ under $f$ is a path that
connects $y_{1}$ to $y_{N}$, passing through the points $y_{2},\dots,y_{N-1}$
(since the segment $[x_{1},x_{N}]$ is mapped by $g$ to a path from
$z_{1}$ to $z_{N}$ on the line, which must pass through $z_{2},\dots,z_{N-1}$).
This implies that the function $f$ maps a path of length $\left\Vert x_{1}-x_{N}\right\Vert \leq\mathrm{diam}(x_{1},\dots,x_{N})$
to a path of length at least $\mathrm{TSP}(y_{1},\dots,y_{N})$, as
a result there must be a point $x$ on the segment $[x_{1},x_{N}]$
whose Jacobian has operator norm at least $\frac{\mathrm{TSM}(y_{1},\dots,y_{N})}{\mathrm{diam}(x_{1},\dots,x_{N})}$.
\end{proof}
Let us now prove that the global minima are approximately rank $k$
in deep networks:
\begin{prop}[Proposition 5 in the main]
Let the `true function' $f^{*}:\Omega\to\mathbb{R}^{d_{out}}$ be
piecewise linear with $\mathrm{Rank}_{BN}(f^{*})=k$, then there is
a constant $C$ which depends on $f^{*}$ only such that any global
minimum $\hat{\mathbf{W}}$ of the loss $\mathcal{L}_{\lambda}(\mathbf{W})=\frac{1}{N}\sum_{i=1}^{N}\left\Vert f_{\mathbf{W}}(x_{i})-f^{*}(x_{i})\right\Vert ^{2}+\frac{\lambda}{L}\left\Vert \mathbf{W}\right\Vert ^{2}$
for a sufficiently wide network satisfies
\[
\frac{R(f_{\hat{\mathbf{W}}};\Omega,\sigma_{a},L)}{L}\leq k+\frac{C}{L}.
\]
\end{prop}
\begin{proof}
The true function $f^{*}$ equals the composition of two piecewise
linear functions $h\circ g$ with $g:\Omega\to\mathbb{R}^{k}$ and
$h:\mathbb{R}^{k}\to\mathbb{R}^{d_{out}}$ which can be represented
by networks of depth $\left\lceil \log_{2}d_{in}+1\right\rceil +1$
(resp. $\left\lceil \log_{2}k+1\right\rceil +1$) and with parameters
$\mathbf{W}_{g}$ (resp. $\mathbf{W}_{h}$) using Corollary \ref{thm:piecewise-linear-func-representation}.
For $L>\left\lceil \log_{2}d_{in}+1\right\rceil +\left\lceil \log_{2}k+1\right\rceil +2$,
consider the network made up of the concatenation of the network representing
$g$, and the network $h$ at the end, with identity layers in the
middle. This concatenated network has parameters norm
\[
\left\Vert W_{g}\right\Vert ^{2}+k(L-\left\lceil \log_{2}d_{in}+1\right\rceil -\left\lceil \log_{2}k+1\right\rceil -2)+\left\Vert W_{h}\right\Vert ^{2}.
\]
Since this network recovers the true function, we have that for any
global minimum $\hat{\mathbf{W}}$:
\begin{align*}
\frac{R(f_{\hat{\mathbf{W}}};\Omega,\sigma_{a},L)}{L} & \leq\frac{1}{L}\left\Vert \hat{\mathbf{W}}\right\Vert ^{2}\\
 & \leq\frac{1}{\lambda}\min_{\mathbf{W}}\mathcal{L}(\mathbf{W})\\
 & \leq\frac{1}{L}\left(\left\Vert W_{g}\right\Vert ^{2}+k(L-\left\lceil \log_{2}d_{in}+1\right\rceil -\left\lceil \log_{2}k+1\right\rceil -2)+\left\Vert W_{h}\right\Vert ^{2}\right)\\
 & =k+\frac{C}{L}.
\end{align*}

The minimal widths required for this result only depends on the decomposition
$f=h\circ g$ chosen, it does not depend on the depth.
\end{proof}

\subsection{Rank of Kernel Ridge Regression}

Consider a translation- and rotation-invariant kernel $K(x,y)=k(\left\Vert x-y\right\Vert )$
then the Kernel Ridge Regression (KRR) predictor with ridge parameter
$\lambda$ and on inputs $X$ and outputs $Y$ is of the form
\[
\hat{f}_{K}(x)=K(x,X)\left(K(X,X)+\lambda I_{N}\right)^{-1}Y.
\]

The Jacobian of $\hat{f}_{K}(x)$ equals $J\hat{f}_{K}(x)=JK(x,X)\left(K(X,X)+\lambda I_{N}\right)^{-1}Y,$
where
\[
JK(x,X)=(X-x)\mathrm{diag}\left(\frac{k'(\left\Vert x-X\right\Vert )}{\left\Vert x-X\right\Vert }\right),
\]
where $X-x$ is the $n_{0}\times N$ dimension matrix with entries
$(X-x)_{ki}=X_{ki}-x_{k}$ and $\mathrm{diag}\left(\frac{k'(\left\Vert x-X\right\Vert )}{\left\Vert x-X\right\Vert }\right)$
is the $N\times N$ diagonal matrix with diagonal entries $\frac{k'(\left\Vert x-x_{i}\right\Vert )}{\left\Vert x-x_{i}\right\Vert }$.
Since $\mathrm{diag}\left(\frac{k'(\left\Vert x-X\right\Vert )}{\left\Vert x-X\right\Vert }\right)$
is invertible, we have $\mathrm{Rank}\left(JK(x,X)\right)=\mathrm{Rank}(X-x)$.
For almost all choices of $x$ (i.e. as long as $x$ does not belong
to a zero Lebesgue measure set) one has $\mathrm{Rank}\left(JK(x,X)\right)=\mathrm{Rank}(X-x)=\min\{d_{in},N,\mathrm{Rank}X+1\}$.

Assuming that $Y$ conditioned on $X$ is sampled from a distribution
with full support (as is the case when there is i.i.d. noise on the
entries of $Y$ for example), then $\mathrm{Rank}Y=\min\left\{ d_{out},N\right\} $
with prob. 1. As a result, the rank of $J\hat{f}_{K}(x)$ will be
$\min\left\{ \mathrm{Rank}(X-x),\mathrm{Rank}Y\right\} =\min\left\{ \mathrm{Rank}X+1,N,d_{out}\right\} $
with prob. 1.

Assuming $N$ to be larger than the input and output dimensions and
$X$ to be full rank, we obtain that the Jacobian $J\hat{f}_{K}(x)$
is almost surely full rank
\[
\mathrm{Rank}\left(J\hat{f}_{K}(x)\right)=\min\{d_{in},d_{out}\}.
\]

\section{Technical Results}

\subsection{Representation of Piecewise Linear Functions}

Let us prove a generalization of the Theorem 2.1 from \cite{arora_2018_relu_piecewise_lin}:
\begin{cor}
\label{thm:piecewise-linear-func-representation}For any $L\geq\left\lceil \log_{2}(n_{0}+1)\right\rceil +1$,
and any piecewise linear function $f:\mathbb{R}^{d_{in}}\to\mathbb{R}^{d_{out}}$
there are widths $\mathbf{n}$ and parameters $\mathbf{W}$ such that
$f_{\mathbf{W}}=f$.
\end{cor}
\begin{proof}
This result was proven in \cite{arora_2018_relu_piecewise_lin} for
ReLU networks, we therefore simply need to show that given a ReLU
network with widths $\mathbf{n}$ and parameters $\mathbf{W}$ such
$f_{\mathbf{W}}=f$, there are widths $\mathbf{n}'$ and parameters
$\mathbf{W}'$ such that $f_{\mathbf{W}'}=f$ for a network with a
nonlinearity $\sigma_{a}$.

Notice that for any $a\in(-1,1)$ $\frac{\sigma(x)-a\sigma(-x)}{1-a^{2}}=\max\{0,x\}$.
By doubling the number of neurons in each hidden layer, i.e. $n'_{\ell}=2n_{\ell}$,
we can represent the same output function $f$ with the nonlinearity
$\sigma_{a}$.
\end{proof}

\subsection{Weak Balancedness Property}

In the analysis of linear networks a widely used tool is the notion
of balancedness, which is an invariant of linear networks during training.
Furthermore any at any local minimum of the $L_{2}$-regularized loss
the weights of the network must be balanced. While no direct equivalent
of this notion exists for nonlinear DNNs, for homogeneous nonlinearities
a weaker notion exists, which we describe now.
\begin{prop}
\label{prop:balancedness}Let $\mathbf{W}$ be a local minimum of
the $L_{2}$-regularized loss $\mathcal{L}_{\lambda}(\mathbf{W})=C(f_{\mathbf{W}})+\lambda\left\Vert \mathbf{W}\right\Vert ^{2}$
for some $\lambda>0$. Then\textbf{ }$\mathbf{W}$ satisfies
\[
\left\Vert W_{\ell}\right\Vert _{F}^{2}+\left\Vert b_{\ell}\right\Vert ^{2}=\left\Vert W_{\ell+1}\right\Vert _{F}^{2}.
\]
\end{prop}
\begin{proof}
Given a local minimum $\mathbf{W}$ of the $L_{2}$-regularized loss,
one can change the the weights of the network to new weights $\mathbf{W}(\alpha_{1},\dots,\alpha_{L})$
with the same outputs for any set of scalars $\alpha_{1},\dots,\alpha_{L}$
such that $\alpha_{1}\cdots\alpha_{L}=1$ :
\begin{align*}
W_{\ell}(\alpha_{1},\dots,\alpha_{L}) & \mapsto\alpha_{\ell}W_{\ell}\\
b_{\ell}(\alpha_{1},\dots,\alpha_{L}) & \mapsto\alpha_{1}\cdots\alpha_{\ell}b_{\ell}
\end{align*}
Since $\mathbf{W}$ is a local minimum, the derivatives of the norm
of the parameters $\left\Vert \mathbf{W}(\alpha_{1},\dots,\alpha_{L})\right\Vert ^{2}$
w.r.t. to $\alpha_{1},\dots,\alpha_{L}$ at $\alpha_{1}=\dots=\alpha_{L}=1$
have to be orthogonal to the constraint space $\alpha_{1}\cdots\alpha_{L}=1$.
At $\alpha_{1}=\dots=\alpha_{L}=1$the normal space (orthogonal to
the tangent space) is the space of constant vectors, since it is spanned
by the gradient of the product $\alpha_{1}\cdots\alpha_{L}$. This
implies that the values
\begin{align*}
\partial_{\alpha_{\ell}}\left(\left\Vert \mathbf{W}(\alpha_{1},\dots,\alpha_{L})\right\Vert ^{2}\right)(1,\dots,1) & =\partial_{\alpha_{\ell}}\left(\sum_{\ell=1}^{L}\left\Vert \alpha_{\ell}W_{\ell}\right\Vert _{F}^{2}+\left\Vert \alpha_{1}\cdots\alpha_{\ell}b_{\ell}\right\Vert ^{2}\right)(1,\dots,1)\\
 & =2\left\Vert W_{\ell}\right\Vert _{F}^{2}+2\left\Vert b_{\ell}\right\Vert ^{2}+\dots+2\left\Vert b_{L}\right\Vert ^{2}
\end{align*}
must all be equal. This equality for two consecutive layers implies
that
\[
\left\Vert W_{\ell}\right\Vert _{F}^{2}+\left\Vert b_{\ell}\right\Vert ^{2}+\dots+\left\Vert b_{L}\right\Vert ^{2}=\left\Vert W_{\ell+1}\right\Vert _{F}^{2}+\left\Vert b_{\ell+1}\right\Vert ^{2}+\dots+\left\Vert b_{L}\right\Vert ^{2}
\]
and therefore that at any local minimum
\[
\left\Vert W_{\ell}\right\Vert ^{2}+\left\Vert b_{\ell}\right\Vert ^{2}=\left\Vert W_{\ell+1}\right\Vert ^{2}.
\]
\end{proof}
\begin{rem}
A yet stronger notion of balancedness can be obtained by observing
that this rescaling of the weights can be done neuron by neuron instead
of layer by layer, but we do not need this notion for our proofs.
\end{rem}

\section{Experimental Setup}

All our experiments were done on fully-connected ReLU networks with
biases. We used diagonal networks, i.e. $n_{1}=n_{2}=\dots=n_{L-1}=w$
for some width $w$. We trained the network using Adam with weight
decay, in some cases we used traditional gradient descent at the end
of training to make sure to converge as close as possible to a local
minimum. When the ridge $\lambda$ is small, we often observe two
phases in learning: in the first phase, the cost $C(f_{\mathbf{W}})$
goes down very fast as the network fits the data, in the second part
the cost remains close to zero and the parameter of the network slowly
goes down. This second part is very slow and we did in most case stop
before the parameter norm had completely stabilized. Note that even
with this `early stopping' we observed results consistent with our
theory.

For Figure 1, the inputs $x\in\mathbb{R}^{50}$ and outputs $y\in\mathbb{R}^{50}$
were generated from a $15$-dimensional latent representation $z\in\mathbb{R}^{15}$
sampled with i.i.d. $\mathcal{N}(0,1)$ entries. The inputs $x$ then
equal $x=g(x)$ for a function $g:\mathbb{R}^{15}\to\mathbb{R}^{50}$
and the outputs equal $y=h(x_{1},\dots,x_{5})$ for a function $h:\mathbb{R}^{5}\to\mathbb{R}^{50}$
which depends only on the first 5 coordinates of the latent space.
Both functions $g,h$ are represented by random shallow ReLU networks
with inner width $n_{1}=100$.

For Figure 2 the data points from the 4 classes were using the same
inverted S-shape distribution and translated on the $x$ axis according
to their class.

For Figure 3, we used the MNIST dataset on the left and on the right
data of the form $g(z)$ for $z$ random 1D Gaussian scalars and a
function $g:\mathbb{R}^{1}\to\mathbb{R}^{2}$ represented y a random
ReLU network.

\end{document}